\documentclass[journal]{IEEEtran}

\usepackage[sort, numbers]{natbib}
\usepackage{amssymb}
\usepackage{tabularx}
\usepackage{graphicx}
\usepackage{subfigure}
\usepackage{pgfplots}
\usepackage{tikz}
\usepackage{pgfplotstable}
\usepackage{algorithm}
\usepackage{algorithmic}
\usepackage{multirow}
\usepackage{amsmath,amsfonts}
\usepackage{array}
\usepackage{stfloats}
\usepackage{fixltx2e}
\usepackage{verbatim}
\usepackage{balance}
\usepackage{stmaryrd}
\def\BibTeX{{\rm B\kern-.05em{\sc i\kern-.025em b}\kern-.08em
    T\kern-.1667em\lower.7ex\hbox{E}\kern-.125emX}}

\begin{document}
\title{FedDriveScore: Federated Scoring Driving Behavior with a Mixture of Metric Distributions}
\author{Lin~Lu,~\IEEEmembership{Member,~IEEE}, 
\thanks{Lin Lu are with the College of Computer and Information Technology, China Three Gorges University, Yichang, 443517, China. }}  

\markboth{Journal of \LaTeX\ Class Files,~Vol.~18, No.~9, September~2020}%
{Lu \MakeLowercase{\textit{et al.}}}

\maketitle

\begin{abstract}
Scoring the driving performance of various drivers on a unified scale, based on how safe or economical they drive on their daily trips, is essential for the driver profile task. Connected vehicles provide the opportunity to collect real-world driving data, which is advantageous for constructing scoring models. However, the lack of pre-labeled scores impede the use of supervised regression models and the data privacy issues hinder the way of traditionally data-centralized learning on the cloud side for model training. To address them, an unsupervised scoring method is presented without the need for labels while still preserving fairness and objectiveness compared to subjective scoring strategies. Subsequently, a federated learning framework based on vehicle-cloud collaboration is proposed as a privacy-friendly alternative to centralized learning. This framework includes a consistently federated version of the scoring method to reduce the performance degradation of the global scoring model caused by the statistical heterogeneous challenge of local data. Theoretical and experimental analysis demonstrate that our federated scoring model is consistent with the utility of the centrally learned counterpart and is effective in evaluating driving performance.
\end{abstract}

\begin{IEEEkeywords}
driver profile, driving behavior scoring, privacy preservation, federated learning, vehicle-cloud collaboration, driving data analysis
\end{IEEEkeywords}

\section{Introduction}
\IEEEPARstart{S}{coring} driving behavior is a common way to assess driving performance of drivers in the driver profile task \cite{ellison2015driver, azadani2021driving}. This task is different from other primary tasks of driving behavior analysis in that its purpose is not just detecting a specific type of behavior or style. Instead, it uses detected risky or economic maneuvers as scoring metrics and then calculates a score that represents the level of competence of the drivers. As a result, drivers could be ranked or compared on a common scale based on how safe or economical they drive. Profiling drivers with scores creates valuable opportunities for applications such as Usage-Based Insurance (UBI)~\cite{handel2014insurance, yin2018advanced}, fleet management~\cite{lu2021bi}. Taking UBI as an example, an insurance company can rank its customers and calculate personalized premiums according to their personal information and the scores recorded in their profile. It can significantly reduce insurance costs for safe drivers and help develop good driving habits for risky drivers.

Today, with the popularity of (intelligent) connected vehicles and the Internet of Vehicles (IoV), it is possible to study naturalistic driving behavior by collecting drivers' driving data sourced from automobile sensors during their everyday trips. Driving data such as velocity, steering angle, braking intensity, and fuel consumption are critical for building the scoring function/models for the above applications. With real-world driving data shared from the driver population, these applications can detect risky driving maneuvers such as sudden acceleration, abrupt lane change, and harsh braking and then develop evaluating metrics. Afterwards, the goal is to build a scoring function/model which is then used to calculate a summarizing score for each driver or his trip.

However, in practical applications, achieving this goal is faced with two challenges: 1) the absence of label information being fed to a regression algorithm for training a scoring model. As recognized, it is costly to obtain already scored samples from real-world driving data, resulting in the unavailability of the labeled information (e.g. risky scores). This makes the driver scoring problem different from the typical regression problem and henceforth requires unsupervised modeling methods. 2) the increasing privacy issue during individual driving data collection and sharing. Current modeling methods usually follow the Centralized Learning (CL) paradigm, where driving data sets are gathered on a central server, resulting in concerns about driver privacy exposure. Suppose that an insurer collects drivers' vehicle data to build a scorecard for UBI purposes; a curious insurer is likely to infer a driver's driving habits, address, income level and even cause data leakage to external adversaries~\cite{zhou2018location}. Worse still, the CL way is now subject to serious privacy restrictions, such as the \textit{GDPR}\footnote{https://gdpr-info.eu/} regulation, or the closely related \textit{guidelines on processing personal data in the context of connected vehicles and mobility-related applications}~\cite{de2021european}. Consequently, there is an urgent need for unsupervised and privacy-protecting scoring solutions.

This paper presents a federated unsupervised scoring solution, named FedDriveScore, to address the two challenges. In this solution, an unsupervised scoring strategy inspired by the distribution mixture\footnote{https://en.wikipedia.org/wiki/Mixture\_distribution} is introduced to tackle the first challenge. The strategy first assigns to each evaluation metric a score related to its Cumulative Probability Density (CPD) and expectation type.
Then, considering the correlation of different metrics, the weighting method CRiteria Importance Through Intercreteria Correlation (CRITIC)~\cite{diakoulaki1995determining} is used to weigh the importance of different evaluation metrics. Finally, the driving score is a mixture density (i.e. the weighted sum of individual score) of all metrics. This strategy is termed CRITIC-DM, which has the advantage of objectively and fairly evaluating/comparing driving performance. 

For the second challenge, Federated Learning (FL)~\cite{MAL-083,TRUONG2022103692} instead of CL is used to implement our CRITIC-DM scoring strategy. In FL settings, connected vehicles can collaboratively update model-related parameters with their driving data, but the data is still kept local. The central server in the cloud coordinates the collaboration process, where the locally updated parameters, rather than the driving data, of the edge vehicles are collected and aggregated globally to produce a federated scoring model. However, as a price to protect privacy, the performance of the federated scoring model suffers from the statistical heterogeneity of the driving data of the driver population. Driving data is kept in different connected vehicles, whose local data are Non-Independently and Identically Distributed (Non-IID) due to the personalized driving behaviors of various drivers. As a result, the utility (i.e., model performance) of the global scoring model is generally affected. Therefore, maintaining the utility consistency between the FL- and CL-based CRITIC-DM models is a key issue. To address this issue, we introduce a Consistently Federated version for CRITIC-DM (represented as CF4CRITIC-DM), aiming to ensure consistency of the federated model to that of CL version. Theoretical analysis and experimental tests of our solution demonstrate that our solution is effective in achieving a lossless and privacy-friendly evaluation of driving performance in the driver profile task.

\section{Related Work}
\newtheorem{definition}{Definition}[section]
\newtheorem{them}{Theorem}[section]
\newtheorem{lema}{Lemma}[section]
\newtheorem{assum}{Assumption}[section]
\newtheorem{proof}{Proof}[section]

This section first reviews previous literature on unsupervised scoring methods for driver profiling and then introduces the preliminary of FL and its application progress in related area. Finally, the literature gap is discussed as the motivation for our proposed idea.

\subsection{Unsupervised Scoring Methods}
In order to narrow the relevant work within the scope of this study, a simple and descriptive definition of driving behavior scoring is given below.

\begin{definition}{
Scoring driving behavior defines the process of using vehicle sensor data to extract evaluation metrics of a trip driven by a (human) driver first. Then, historical trip metrics collected from a driver population are later used to formulate a scoring model, which calculates a profile (a score) for every driver for downstream applications.
}\end{definition}

Mathematically, given $\boldsymbol{x}\in \mathbb{R}^d$ as a set of metrics used to evaluate driver driving performance and letting $\mathcal{F}$ be the target scoring model, our goal can be simply formulated as $\mathcal{F}\{\boldsymbol{x}_1, ..., \boldsymbol{x}_d \}$ $\rightarrow y$, where $y$ is the final score. By doing so, different driver's driving performance can be quantitatively mapped into a comparable space, where higher scores are associated with better driving performance, while poorer driving performance is assigned a lower score. As mentioned, it is easy to obtain $\boldsymbol{x}$ but difficult to obtain the corresponding labeled score $y$ for everyday driving trips. To this end, several studies have proposed quantitatively measuring driving performance in an unsupervised manner.

To begin with, the established evaluation metrics shall be evident enough for profiling drivers with their probability of crash risk and/or eco-driving skill. Recently, \citet{SINGH2021106349} presented a systematic review of studies on driver behavior profiling. This review revealed different parameters adopted by researchers to profile drivers, such as vehicle speed, acceleration, braking, mileage, fuel consumption, etc. These parameters are usually used to generate two types of evaluation metrics. One is the frequency of critical driving events related to the total distance / time traveled~\cite{li2011trip}. For example, the number of rapid accelerations per km can be obtained by dividing the number of rapid accelerations in a driving trip by the distance traveled. The other type is the ratio of events related to the total driving distance/time. For example, the idle ratio is obtained by dividing the idle duration in a driving trip by its total duration. Next, by establishing a fixed number of metrics, trips of different lengths from different drivers can be transformed into the same dimension.

\citet{castignani2015driver} proposed a mobile phone scoring application: SenseFleet for vehicle-independent driver profiling. In SenseFleet, any single trip is scored with a value between 0 and 100 (with 100 being the best possible score). When a trip starts, a set of fuzzy logic was defined to detect and count different risky events of different types (e.g. speeding, hard acceleration, hard deceleration, and aggressive steering). Then, for each combination of type and severity of an event, the system reduces the score by a predefined number of points. As specified in their previous work~\cite{castignani2013driver}, they calculated the center of gravity of the fuzzy output curve obtained as a score, which was then used to rank the different drivers. It can be noted that the effectiveness of this method relies heavily on the pre-defined rules in the fuzzy system.

\citet{6856461} also introduced a mobile phone application "DriveSafe" to score driver driving behavior and alert them when they behaved unsafely. Their program first detects acceleration, braking, and turning events for each trip. The product of the number and intensity of a driving event divided by the distance traveled is then taken as a metric (named indicator in this study).
Next, each indicator is modeled as a Gaussian distribution, and its Cumulative Density Function (CDF) is used to get a score for a specific indicator value. Let $\boldsymbol{x}_i$ be a specific value of the $i^{th}$ metric, then it will be scored as 1-CDF$_i({\boldsymbol{x}_i})$, where CDF$_i({\boldsymbol{x}_i})$ is the Cumulative Probability Density (CPD).  
At the beginning of each trip, drivers get 10 points, and, depending on their driving behavior, they were penalized by deducting the scores for each driving event they perform per km. The final per-trip score is the average of the 3 indicator scores. Later in 2019, they illustrated the potential application of "DriveSafe" to score a group of old age drivers using naturalistic driving data obtained from mobile phone sensors~\cite{bergasa2019naturalistic}. In their study, the distribution nature of the indicators was used instead of the subjective rules used by~\citet{castignani2013driver}, but simply averaging the different indicator scores neglects the underlying correlation and importance of these indicators.

A similar term to driving behavior scoring is driving style scoring or recognizing driving style into continuous index~\cite{martinez2017driving}. \citet{zgl2018AFU} proposed a fully unsupervised driver scoring framework that attempts to assign a driving score based on the probability of the occurrence of different driving styles according to the geometry of the road topology and traffic conditions. The basis of their study was the dependence of normal driving patterns on road geometry (trajectory) and traffic flow characteristics. They first learned these road-dependent driving norms and assigned higher scores to drivers that conform to the norms. The scoring strategy involves simple joint probabilities estimated by the frequency of samples that lie in preclustered driving styles.
\citet{mohammadnazar2021classifying} first extracted the driving segments of drivers on different types of roads, and then used the K-Means algorithm to group these segments to obtain three styles, marked 1 to 3, respectively. Then, to quantitatively evaluate the change in driving style of the same driver, the driving score is defined as the average value of the number of three driving styles that occurred for each type of road. Recently, \citet{schoner2021safety} presented an objective scoring system that assigns safety scores to the observed driving style and aggregates them to provide an overall safety score for a given driving session(trip). The safety score was developed by matching the safety indexes with the maneuver-based parameter ranges processed from real-world highway traffic data. For each parameter, the ranges of normal parameters were identified first by statistical analysis, and its allocation to the safety index spread (0 to 1) was based on driving physics. The assessment of a complete drive in a summarizing score is based on a time-dependent recording and weighted averaging of the scores of the single maneuvers. These methods distinguish themselves in that the driving styles shall be clustered first, which limits their scalability.

To alleviate subjective dependence on weights and scoring rules in previous scoring methods, \citet{liu2017driving} adopted the Entropy Weight Analytical Hierarchy Process (EW-AHP) to assign objective weights to the different evaluation metrics to score drivers based on their driving behavior. In this study, various driving factors such as mileage, driving time, traffic flow, speeding, traffic violations, hard deceleration and acceleration, and severe maneuvers were used to rate drivers based on driving safety. The Entropy Weight (EW)~\cite{deng2000inter} calculates their weights from the perspective of information entropy. However, the AHP method requires manual comparison and weighting of different factors, which makes it susceptible to subjective personal biases. In view of this, we proposed an improved scoring approach~\cite{lu2021bi} to handle large-scale driving trips. This method also adopted the EW method but relied on the CDF of metrics (similar to the idea of \cite{6856461}) to give a weighted sum, which is the final score. It totally uses the statistical information of the big driving data, which yields strong objectiveness. However, the correlation between evaluation metrics is ignored by the EW method, which motivated us to present a new unsupervised scoring method, namely the CRITIC-DM mentioned above.

\subsection{Federated Learning}
Federated learning~\cite{MAL-083, TRUONG2022103692} highlights itself on collaborative machine learning tasks in a decentralized environment. In a decentralized environment, private data sets are scattered among various devices or silos (referred to as clients uniformly). Clients can establish connections to a central server, which orchestrates the entire training process and updates the global model. Instead of sharing their private data, the clients influence an overarching model through various parameter updating techniques (typically local training and global aggregation). In this way, FL seeks to build models that can benefit from everyone's data without being exposed. 
Moreover, it is robust to device offline cases and can handle the imbalanced distribution of local data samples.\footnote{https://en.wikipedia.org/wiki/Federated\_learning}

Given a total of $K$ clients that each has a subset of samples belonging to a homogeneous metric space $X\in \mathbb{R}^{n\times d}$, FL aims to train an aggregated global model, by optimizing the following Eq.~\ref{eq:fl}:
\begin{equation}
\label{eq:fl}
\begin{aligned}
\min &\quad L(X;\mathcal{\hat{F}})=  \mathop{\mathbb{E}}\limits_{\forall X^{(i)} \subset X} [L(X^{(i)};\mathcal{\hat{F}})] \\
& \mathcal{\hat{F}} = Agg(\mathcal{\hat{F}}_1,...,\mathcal{\hat{F}}_K) &
\end{aligned}
\end{equation}
where $X^{(i)}$ is the feature sample set of client $i$, and $\mathcal{\hat{F}}_i$ is the local model trained on $X^{(i)}$. $L$ is the training loss, $Agg$ is some aggregation strategy, such as FedAvg~\cite{mcmahan17a}, to produce the global model $\mathcal{\hat{F}}$.

FL achieves privacy preservation by keeping local datasets inaccessible to others. However, as a price, the accuracy of the learned model is somewhat sacrificed. According to the following definition given by~\cite{yang2019federated}, there is a utility/performance gap between the FL model and the conventionally non-federated model. When choosing FL to build a CRITIC-DM based scoring model, our primary goal is to allow its FL version to be approximate or equal to the performance of its CL version.

\begin{definition}[$\delta$-Accuracy]{
Assume that $\mathcal{F}$ is a model trained on a centralized data set $X=\{ X_1 \cup, ...,\cup X_K \}$, the accuracy of which is represented by $h$. The accuracy of the FL model $\hat{\mathcal{F}}$, denoted $\hat{h}$, should be very close to the accuracy of $\mathcal{F}$, denoted $h$. Let $\delta$ be a non-negative real number; then the federated learning algorithm has a loss of accuracy $\delta$ if $|h - \hat{h}|<\delta$.
}\end{definition}

FL in IoV has attracted increasing attention today. 
For example, \citet{elbir2020federated} designed a distributed training framework for FL-based Machine Learning (ML) models as an effective learning solution for vehicular networks and edge intelligence, in contrast to classical ML techniques based on centralized training on cloud servers. \citet{9360666} proposed a conceptual framework for vehicle-oriented networks: FVN (Federated Vehicular Network). They focused on the use of FVN to support FL. \citet{9205482} reviewed and summarized recent FL applications in vehicular networks.

Specifically, for the privacy-preserved analysis of driving behavior, \citet{rizzo2015privacy} considered that commercial insurance companies use private driving data (e.g., acceleration events, average acceleration, braking events) to predict driving style, which infringes on personal privacy. Therefore, they proposed vertical FL for insurance companies and connected vehicles to jointly train a secure decision tree based on historical vehicle travel data without compromising their privacy. In their work, the Paillier algorithm for homomorphic addition is used to implement a secure summing during the decision tree training process. \citet{chhabra2023privacy} applied FL to driver behavior analysis in which connected vehicles in the network collaborate to train CNN-LSTM and CNN-Bi-LSTM deep learning models for driver behavior classification (safe, unsafe, or fatigue) without sharing raw data. To overcome the lack-of-label problem faced by supervised FL, our recent work~\cite{lu2023federated} focused on unsupervised FL to recognize driving styles from drivers' private trajectories. In this study, a federated version of the K-means algorithm was proposed to classify drivers into 3 different styles. 
To our knowledge, the work considering scoring driving behavior under the prospect of privacy of the driver is very limited. 

\subsection{Knowledge Gap}
It should be noted that for existing scoring methods, it is difficult to compare those unsupervised scoring methods and argue which is better. We previously presented a scoring method based on the EW and metrics' distributions to achieve fairness and objectiveness to a large population of drivers. However, the correlation of different metrics is not considered by EW, resulting in unreasonable weights assignment. To solve this problem, the CRITIC~\cite{diakoulaki1995determining} method was further adopted, incorporating the inter-correlation of metrics. As a result, this study presents a new method for scoring driving behavior, where each evaluation metric is modeled as a distribution form, and the CRITIC method is used to calculate the weights of different metrics. Then the CRITIC weighted sum of the CPD-based metrics' scores is taken as the summarizing score.

The preservation of privacy is another neglected area in the field of driver profile. FL under the IoV is a promising way to perform driver profiling with privacy preserved. Moreover, related works~\cite{rizzo2015privacy, lu2023federated} also used Homomorphic Encryption (HE), specifically the Paillier crytosystem, to protect model parameters. Therefore, we later proposed an HE-based FL framework to build the above CRITIC-DM model. In short, the lack of consideration of both the unsupervised and the privacy-preserving driving behavior scoring problem motivated this study. 

\section{Problem Formulation}
Generally speaking, scoring driving behavior is to measure the driving performance of a group of drivers based on metrics related to safety, economy, and convenience. We consider a set of historical driving trip data from this driver group, each trip will generate a vector of $d$ evaluation metrics (denoted as $\boldsymbol{x}$), and each metric has a predetermined qualitative goal of expectation (e.g. the bigger, the better). For example, the frequency of risky maneuvers normalized by trip distance could be treated as evaluation indicators, and surely their expectations are the lower, the better. 

For the problem of how to obtain the final score, we assume that the driving score $y$ is the outcome of mixture densities from $d$ metrics. This is inspired by the assumption of mixture distribution. With this assumption, a random variable's cumulative distribution function (and the probability density function if it exists) can be expressed as a convex combination (i.e. a weighted sum, with non-negative weights that sum to 1) of other distribution functions and density functions. The individual distributions that are combined to form the mixture distribution are called mixture components, and the weights associated with each component are called mixture weights.

Specifically in this study, our problem is transformed to find a mixture model $\mathcal{F}$ depending on CDFs, weights, and also the expectation types of evaluation metrics. Formally, the score of $\mathcal{F}(\boldsymbol{x})$ is defined by Eq.~\ref{eq:wsm}, where $\boldsymbol{w}$ is the mixture weights, $\top$ is the vector transformation operation. $\boldsymbol{q}$ is a metric-level score vector, which is given by metric-level score functions. $f_i(\cdot)$ is the $i^{th}$ metric-level score function dependent on its PDF / CDF and expectation type.

\begin{equation}
\label{eq:wsm}
\begin{aligned}
y=\mathcal{F}(x) &= \boldsymbol{w}^{\top} \cdot \boldsymbol{q} = \sum_{i=1}^d \boldsymbol{w}_i \cdot f_i(\boldsymbol{x}_i ) \\
&\quad \boldsymbol{x} \in X \\
s.t.& \quad \boldsymbol{q}_i =  f_i(\boldsymbol{x}_i) \in [0, 1] \\
& \quad \boldsymbol{w} >= 0, \sum_{i=1} \boldsymbol{w}_i =1
\end{aligned}
\end{equation}

The expectation types are predetermined, so the objective of building $\mathcal{F}$ is to find reasonable weights and scoring functions at the metric level. Reasonability here means objectivity and effectiveness, which respectively refer to the ability to avoid subjective biased guidance on metric weight and function, and to correctly distinguish and rank the quality of driving performance.

In the FL setting, driving data need not be collected, but kept locally in different vehicles. Let $X=\{X^{(1)}, ..., X^{(K)}\}$ be the local data set from $K$ vehicles. Our goal with FL is to train a federated scoring model (denoted $\hat{\mathcal{F}}$) without revealing or sharing the privacy of clients. Formally, this objective is defined as following:
\begin{equation}
\label{eq:fl_scoring}
\begin{aligned}
&\qquad \mathcal{\hat{F}}=\boldsymbol{w}^{\top} \cdot \vec f(x) \\
&s.t. 
\begin{array}{c}
\boldsymbol{w}=Agg_w(\boldsymbol{w}^{(1)},...,\boldsymbol{w}^{(K)}) \\
f(x)= Agg_f(f^{(1)},...,f^{(K)}) \\
\end{array}
\end{aligned}
\end{equation}
where $\boldsymbol{w}^{(k)}$ and $f^{(k)}$ come from client $k$. $\vec f=[f_1(\cdot),...,f_d(\cdot)]$ is a function vector that contains $d$ metric-level scoring functions. $Agg_w$ and $Agg_f$ are aggregating policies to produce global weights and scoring functions oriented to metrics. 

A common aggregation strategy is FedAvg~\cite{mcmahan17a}, in which the global scoring model is fit per client to produce local updates, then local updates are averaged to be new global model parameters. However, a drawback of FedAvg is that its utility suffers from the Non-IID of clients' data. In this study, due to the personalized driving behaviors of various drivers, there exists the following situation.

\textbf{Metric distribution skewness}: Local metric data for each vehicle are drawn from a different distribution and cannot be seen as a representative sample of the overall distribution. That is, for vehicle $i$ and $j$, there exists $\mathcal{P}(X^{(i)})\neq \mathcal{P}(X^{(j)}) \neq \mathcal{P}(X)$. Such a skewness results in statistical heterogeneity of evaluation metrics, where the statistics of the local metric data are biased estimates of $X$. Simply averaging local metric statistics to be the global metric statistics is not a wise choice. 

As a result, with the statistical heterogeneity of the evaluation metrics, maintaining the utility consistency of the federated scoring model with the CL scoring model is the key issue for our solution. To make the performance of $\hat{\mathcal{F}}$ approximate that of the model $\mathcal{F}$ as possible, the ultimate optimization direction is defined as follow: 
\begin{equation}
\textbf{min}\ \delta \triangleq \textbf{min}\ \mathbb{E}[\mathcal{\hat{F}}(X) - \mathcal{F}(X)] 
\end{equation}

By solving this problem, we can achieve strong consistency between both models, thus ensuring the objectiveness and effectiveness of the federated scoring model. The proposed solution will be introduced in the next section.

\section{Proposed Solution}
\label{sec:method}
In view of the above goals and problems, this section first presents a FL framework, under which we then
introduce the CF4CRITIC-DM method for the training and prediction of the scoring model. The main steps of this method include the definition of evaluation metrics, secure observation of the metric distribution, model training, and scoring processes.

\subsection{FedDriveScore Framework}
\label{sec:sys_arch}
Standing on the IoV infrastructure, our FedDriveScore framework consists of three roles: a coordinator, an arbiter, and a group of clients(vehicles). The coordinator is a cloud server responsible for coordinating the whole task, the clients hold their private data and respond to the server's requests, and the arbiter only provides auxiliary functions like providing homomorphic encryption keys. Fig.~\ref{fig:sys_overview} illustrates the architecture of the proposed FedDriveScore solution, aiming at cross-vehicle FL for the CRITIC-DM model. The workflow involves three stages:

\begin{figure}[!htpb]
\centering
\includegraphics*[width=3.5in]{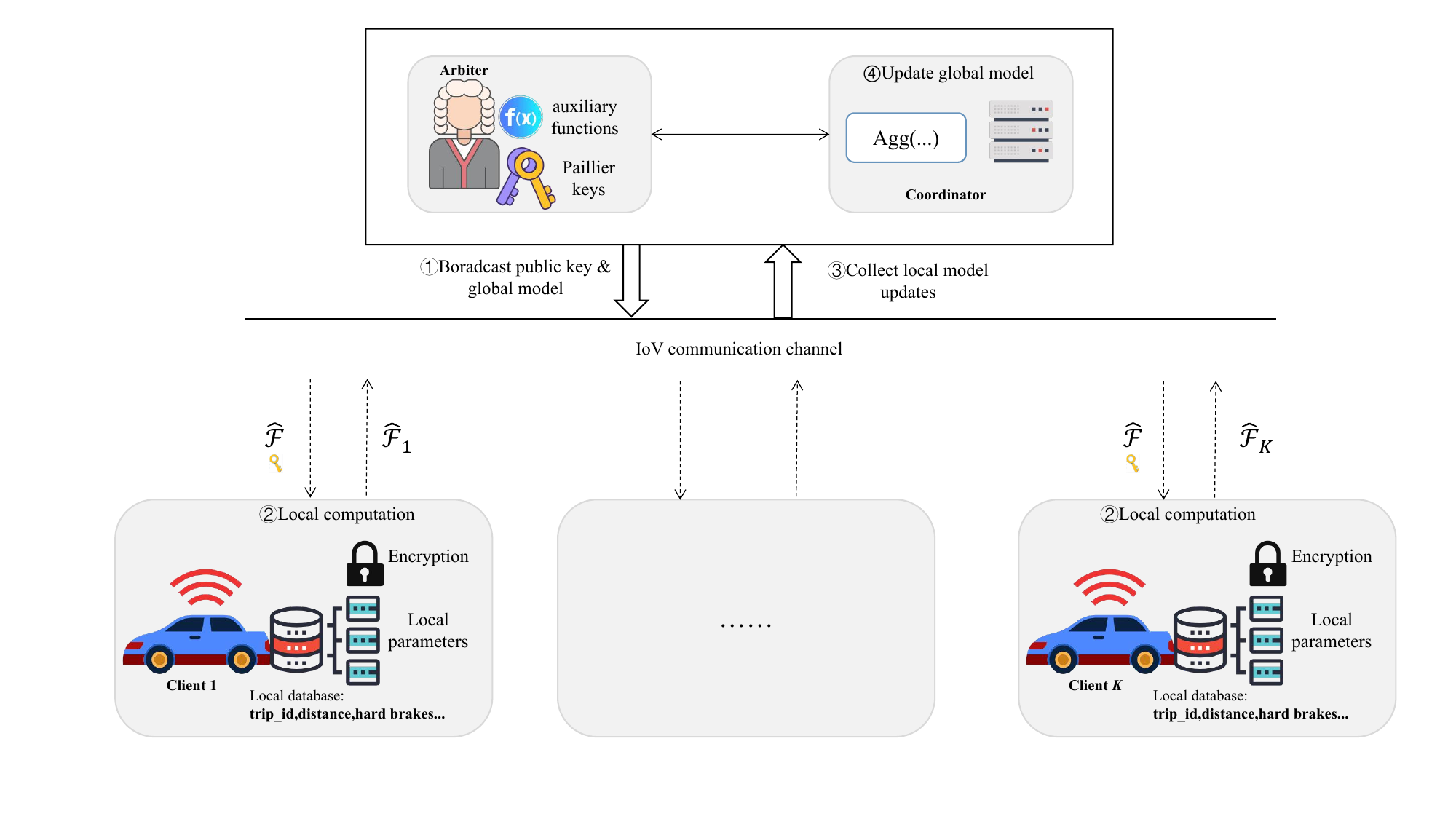}
\caption{The overview of FedDriveScore Framework via IoV.}
\label{fig:sys_overview}
\end{figure}

The first stage is the \textbf{Setup} stage, which is mainly used for the essential parameters and cryptographic keys for all entities. At the beginning, each client pre-installs a pair of public-private keys for signature and authentication. The model owner chooses a server that can communicate with clients and launch this collaborative modeling job. The arbiter, as a trusted third party, is deployed to publish the credible HE-based public key and functions for the clients and the server. The participating roles will verify each other using digital signature technology. 

Then, they enter the \textbf{Collaboration} stage to update the parameters of the global model. In this stage, each client updates the server's global model with its data and returns the required but encrypted parameters. Next, the server aggregates those encrypted parameters of the clients with the help of the arbiter and again broadcasts the updated global model to the clients for the next iteration. Finally, in the model \textbf{Inference} stage, the server sends the global model to all clients to obtain only the \textit{id-score} pairs of the whole population.

\subsection{CF4CRITIC-DM method}

\subsubsection{Definition of Evaluation Metrics}
Defining the evaluation metrics is an important prerequisite to transform variable-length driving data into a fixed-length evaluation space. For example, daily trips of drivers may last from minutes to hours. Therefore, we should decide the number of scoring metrics and determine which are related to safety, economy, and other criteria. The concerned expert usually provides the number and extraction rule to obtain evaluation metrics from the driving data. In doing so, a corresponding metric vector is generated for each driving trip.

In addition, different evaluation metrics have their own types of expectations to promote safe and economical driving. We summarize three types that can be used to evaluate driving behavior, as shown in Table~\ref{tab:criteria_types}. Among them, the positive type indicates that the larger the value of the metric, the better the performance; the negative type indicates that the smaller the value, the better the performance. If a metric is an oscillating indicator, then a low or high value will result in poor performance. As can be seen, these metrics provide accurate information on how drivers drive, but different metrics may have conflict goals. In the experimental section, we will list specific evaluation metrics and their expectation types for different data sets.

\begin{table}[!htbp]\small
\caption{The expectation types of metrics and their qualitative objectives on evaluating driving performance.}
\centering
\begin{tabularx}{0.49\textwidth}{|>{\setlength{\hsize}{0.2\hsize}}X|>{\setlength{\hsize}{0.27\hsize}}X|>{\setlength{\hsize}{0.53\hsize}}X|}
\hline
\textbf{Expectation Type} & \multicolumn{2}{l|}{\textbf{Example}} \\ \hline
Positive & average speed & The higher the average speed, the better the convenience to complete a trip \\ \hline
Negative & hard braking frequency & The fewer frequency of sudden braking, the better the safe driving performance\\ \hline
Oscillator & average engine speed & The lower or higher the average speed, the more it will deviate from the economical speed area, so the economical performance will be worse \\ \hline
\end{tabularx}
\label{tab:criteria_types}
\end{table}

\subsubsection{Observation of Metric Distributions}
In the context of FL, the extraction of metrics is one of the client's local data preprocessing processes. Assume that $D^{(i)}$ represents the data from the driving trip in vehicle $i$, who will convert it into the corresponding metric data set $X^{(i)}$ according to the pre-defined metric extraction program. The program is distributed by the cloud server. However, the local metric data set is only a partial description of the overall driving performance of the driver and cannot reflect the global distribution of metrics. Therefore, the distribution of each metric on the overall data $X$ should be and can be observed jointly using the secure histogram method.

The purpose of the secure histogram is to observe the distributions of metrics via their global histogram plots shown in Fig.~\ref{fig:sec_hist}, without the local data of the clients being collected and exposed. As we can see from Fig.~\ref{fig:sec_hist}, the clients only return the bin counts of each metric to the server. The server accumulates the local bin counts of all metrics and visualizes them in the form of a histogram. The probability density function of different metrics can then be determined through the histogram, such as the exponential distribution and the normal distribution shown in the right plot in Fig.~\ref{fig:sec_hist} . These distributions are usually instantiated based on statistics such as the sample mean and standard deviation. Therefore, one goal of our FL method is to estimate the global descriptive statistics of all evaluation metrics.

\begin{figure}[!htbp]
\centering
\includegraphics*[width=2.9in]{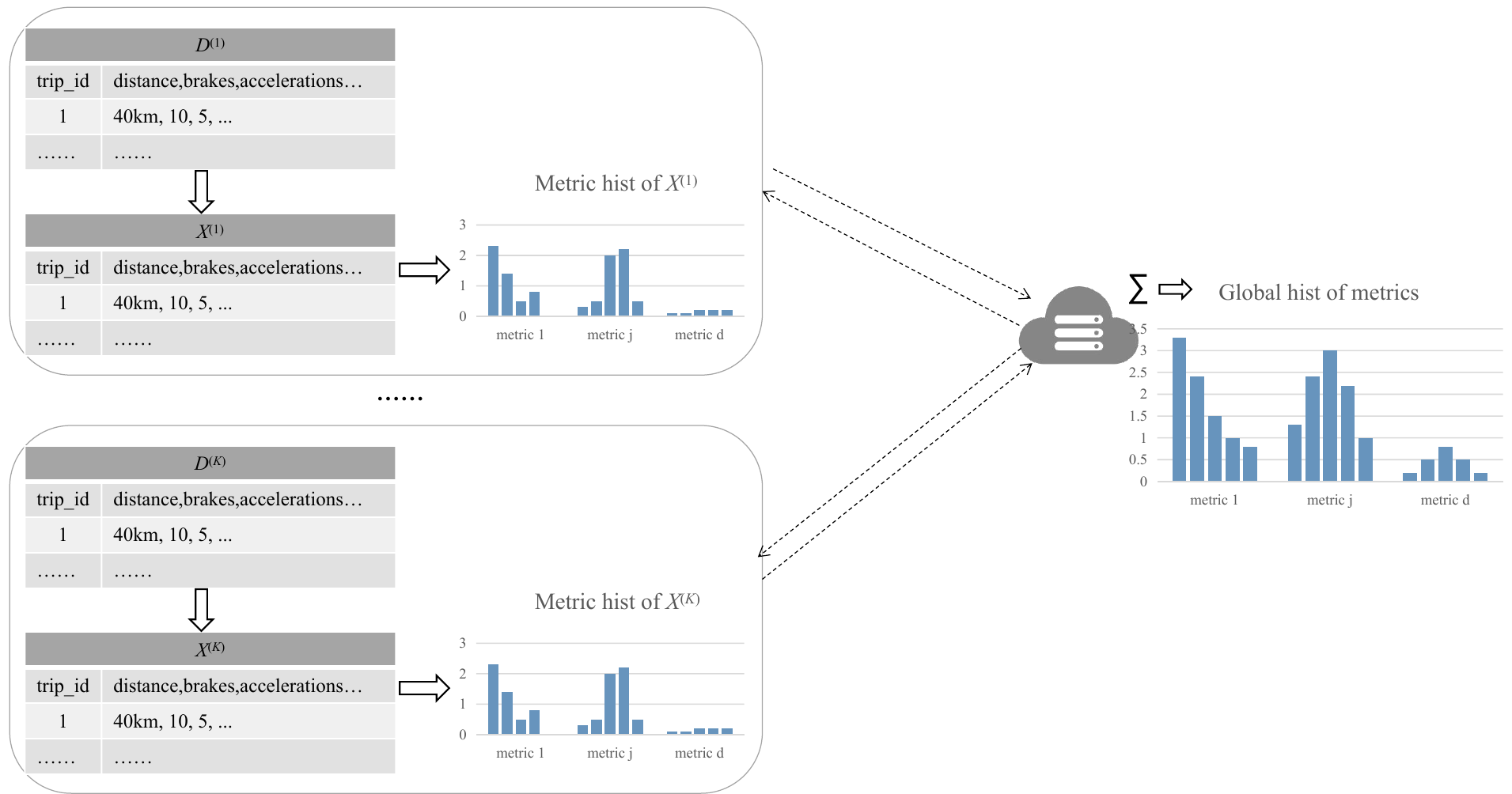}
\caption{Illustration of secure histogram for observing metric distributions.}
\label{fig:sec_hist}
\end{figure}

Next, we establish the link of the metric-level scoring functions to their CDFs. This is because the cumulative density of any distribution is between 0 and 1, and a traditional scoring system should always output a score that can be scaled to between 0 and 1. Therefore, the CPD can be used to establish the following correlation: (1) to maximize the positive metric, our method assumes that a higher CPD means that the driver is more skilled and his driving performance is better; (2) for an oscillatory metric, a value close to its expected mean $\mu$ indicates better driving performance; and (3) for a negative metric, a lower CPD represents better driving performance. This correlation determines that the form of the function $f$ is to obtain the CDFs of all metrics.

\subsubsection{Model Training Process}
Before continuing, we define some variable notation. Specifically, we let $G = \{n, \boldsymbol{s}, \boldsymbol{o}, \boldsymbol{u}, \boldsymbol{v}\}$ denote the number of samples, the vector of the sum of metrics, the vector of the squared sum of metrics, and the vectors of the maximum and minimum values of metrics, respectively. $\boldsymbol{\mu}$ and $\boldsymbol{\sigma}^2$ are the mean and variance vectors of the metrics, respectively. Additionally, let $\llbracket{\cdot}\rrbracket$ be the HE representation of any variable.

\begin{algorithm}[htbp]\small
\caption{The CF4CRITIC-DM algorithm}
\label{alg:flscorecard}
\begin{algorithmic}[1]
\STATE Clients prepare inputs: $X=\{X^{(1)} \cup, ...,\cup X^{(K)}\}$;
\STATE Initialize the arbiter, distributes the Paillier public key;
\STATE Initialize server, and the maximum training iterations $T$, clients selection rate $\tau$. 
\WHILE{$t<T$}
\STATE Server selects a fraction of clients (denoted as $I^t$) and broadcast $G^t$
\FOR{each selected client $i$ in parallel}
\STATE Compute local summary items, denoted as $G_i$.
\STATE Encrypt these items, and return $\llbracket{G_i}\rrbracket = \{ \llbracket{n}\rrbracket,\llbracket{\boldsymbol{s}}\rrbracket, \llbracket{\boldsymbol{o}}\rrbracket,
\llbracket{\boldsymbol{u}}\rrbracket, \llbracket{\boldsymbol{v}}\rrbracket \}_i$ to the server
\ENDFOR
\STATE The coordinator aggregates local statistical items to obtain $\llbracket{G^{t+1}}\rrbracket$
\STATE Coordinator sends $\llbracket{G^{t+1}}\rrbracket$ to the arbiter for computing $\boldsymbol{\mu}^{t+1}, \boldsymbol{\sigma}^{t+1}$
\STATE Coordinator broadcasts $G^{t+1}, \boldsymbol{\mu}^{t+1}, \boldsymbol{\sigma}^{t+1}$ to all clients;
\FOR{each client $i$ in selected clients}
\STATE Resets local statistics, i.e. $G_i^{t+1}=G^{t+1}$;
\STATE Normalize $X^{(k)}$ to unity with $\boldsymbol{u}, \boldsymbol{v}$;
\STATE Compute the local covariance matrix, denoted as $A_i^t$; 
\STATE Return $A_i^t$ to the coordinator
\ENDFOR
\STATE Coordinator aggregates the local covariance matrices to generate $A^{t+1}$
\STATE Coordinator updates the global weights $\boldsymbol{w}^{t+1}$
\ENDWHILE
\STATE Coordinator receives the final model parameters: $\boldsymbol{\mu}^T$, $\boldsymbol{\sigma}^T$, $\boldsymbol{w}^T$
\FOR{metric $j=1$ to $d$}
\STATE Formalize the CDF of $F_j$ according to the distribution type of $j$
\ENDFOR
\STATE Output $\hat{\mathcal{F}}$ that consists $\boldsymbol{w}$ and $\vec f$
\end{algorithmic}
\end{algorithm}

The federated training process of the CRITIC-DM based scoring model is present in Algo.~\ref{alg:flscorecard}, which mainly includes the initialization phase, the vehicle-cloud collaborative computing phase, and the scoring model construction phase. The first three steps of Algo.~\ref{alg:flscorecard} are performed during the initialization phase. Clients prepare their local metric samples and reach a consensus on the chosen arbiter. Then, the arbiter generates a pair of Paillier public and private keys (in the experiment, the key length is 1024 bits) and broadcasts the public key to each client. The arbiter holds its private key, but provides some auxiliary functions, such as performing comparison operations, to other parties that need to decrypt the intermediate parameters first. The coordinator in the initial step ($t$=0) will set $n$ to 0, and set $\boldsymbol{s}^0$ and $\boldsymbol{o}^0$ to a vector of zeros and $\boldsymbol{u}^0$ a vector of very large numbers, and vice versa for $\boldsymbol{v}^0$.

The algorithm then enters the vehicle-cloud collaborative computing stage. In each round of federated communication in the outer loop, the coordinator randomly selects a fraction of clients (denoted as $I^t$) and then coordinates two tasks. The first task (steps 6 to 12) aims to update the global statistics of the metrics, which helps to determine the PDF/CDF parameters of all the metrics. During the update process, the client's local statistical items are in encrypted form to resist server curiosity. The next task (Steps 13 to 20) aims to update the CRITIC weights of the metrics. The two tasks are alternatively repeated for $T$ rounds, and the coordinator receives the optimal global mean, standard deviation, and weight vectors of the metrics.

Steps 22 to 25 describe the scoring model construction phase. For the distribution type observed in the above section, the coordinator instantiates a PDF for each metric and then outputs a CDF vector, denoted as $\Phi = [\Phi_1, ..., \Phi_i,..., \Phi_d]$. This CDF vector, together with the weight vector, can then be used to construct the desired global scoring model $\mathcal{\hat{F}}$. Finally, the coordinator broadcasts the scoring model to all vehicles for the model scoring process.

In order to instantiate the distribution functions for all metrics, we must estimate their PDF parameters, usually mean, variance or standard deviation. On the other hand, the CRITIC method also requires maximum and minimum values for data normalization. Therefore, the first task is to estimate the global statistical items (minimum, maximum, mean and variance) of each metric. If the vehicle responds to the request from the cloud server, it will calculate five local statistical items locally and return them to the server. Before returning, these local statistics will be encrypted to be cipher-text using the Paillier public key. The coordination server can still aggregate these cipher-texts, and Paillier's additive homomorphic encryption has little impact on the performance of the model. 

Suppose $S^t, O^t, U^t, V^t$ is a collection of local statistical items collected by the coordination server in round $t$, it will follow the federation protocol to aggregate them according to Eq.~\ref{eq:agg_sum} (that is, step 12 in Algo.~\ref{alg:flscorecard}) to update the global statistical items. In this equation, $n_i$ is the sample number, $S^t(i,j)$ 
 and $O^t(i,j)$ respectively denote the sample sum and the sample square sum of the $j^{th}$ metric of the $i^{th}$ client in $I^t$.

\begin{equation}
\label{eq:agg_sum}
\begin{aligned}
\llbracket{n^{t+1}}\rrbracket = \llbracket{n^{t}}\rrbracket + \sum_{i \in I^t} \llbracket{n_i}\rrbracket \\
\llbracket{\boldsymbol{s}^{t+1}_j}\rrbracket = \llbracket{\boldsymbol{s}^{t}}\rrbracket + \sum_{i \in I^t} \llbracket{S^t(i,j)}\rrbracket\\
\llbracket{\boldsymbol{o}^{t+1}_j}\rrbracket = \llbracket{\boldsymbol{o}^{t}}\rrbracket + \sum_{i \in I^t} \llbracket{O^t(i,j)}\rrbracket\\
\end{aligned}
\end{equation}

Meanwhile, the coordinator updates the global maximum and minimum values of all metrics. For metric $j$, according to Eq.~\ref{eq:agg_compare}, the maximum value and minimum value of any metric can be obtained. Specifically, the arbiter receives and decrypts the difference value and returns whether it is greater than 0 or not to the coordinator.
\begin{equation}
\label{eq:agg_compare}
\begin{aligned}
\llbracket{\boldsymbol{u}^{t+1}_j}\rrbracket = \llbracket{U^t(i,j)}\rrbracket \quad if \ \llbracket{\boldsymbol{u}^{t}_j-U^t(i,j)}\rrbracket < 0 \quad for\quad i \in I^t \\
\llbracket{\boldsymbol{v}^{t+1}_j}\rrbracket =\llbracket{V^t(i,j)}\rrbracket \quad if \ \llbracket{\boldsymbol{v}^{t}_j - V^t(i,j)}\rrbracket > 0 \quad for\quad i \in I^t \\
\end{aligned}
\end{equation}

It should be noted that Agg$_f$ in Eq~\ref{eq:fl_scoring} refers to the parameters required to aggregate local summaries to update PDF. So, after updating global statistical elements, the mean vector can be calculated simply by $\boldsymbol{\mu}=\boldsymbol{s}^{t+1}/(n^{t+1}-1)$, and the variance vector can be calculated element-wise by $\boldsymbol{\sigma}^2=\boldsymbol{o}^{t+1}/(n^{t+1}-\boldsymbol{\mu}\cdot \boldsymbol{\mu}^{\top})$. Depending on the type of distribution of the $j^{th}$ metric, it can be used to describe a normal distribution PDF or an exponential distribution PDF, etc.

On the client side, global $\boldsymbol{u}^{t+1}$ and $\boldsymbol{v}^{t+1}$ are first calculated through the items in $G^{t+1}$ to normalize $X^{(i)}$ according to Eq.~\ref{eq:normalize}. Note that global $\boldsymbol{\mu}^{t+1}$ and $\boldsymbol{\sigma}^{t+1}$ should also be normalized accordingly.
\begin{equation}
\label{eq:normalize}
Z^{(i)}=(X^{(i)} - \boldsymbol{v}^{t+1})/(\boldsymbol{u}^{t+1} - \boldsymbol{v}^{t+1})
\end{equation}

Then the local covariance matrix of $i^{th}$ client, denoted as $A^{(i)}$, is updated according to Eq.~\ref{eq:covariance}
\begin{equation}
\label{eq:covariance}
A^{(i)}=\frac{1}{n^t-1}(Z^{(i)} - \boldsymbol{\mu}^{t+1})(Z^{(i)} - \boldsymbol{\mu}^{t+1})^{\top}
\end{equation}

On the server side, the local variance matrices are summed up. The global covariance matrix is represented as $A^{t+1}$, whose element $A^{t+1}(j,k)$ is obtained by Eq.\ref{eq:global_cov}:
\begin{equation}
\label{eq:global_cov}
A^{t+1}(j,k)=\sum_{i \in I^t} A^{(i)}(j,k)
\end{equation}

Then, this global variance matrix is divided element-wise by $\boldsymbol{\sigma}$ to produce the Pearson correlation matrix. This correlation matrix is denoted as $R^{t+1}$. The correlation coefficient $R^{t+1}(j,k)$ between metric $j$ and $j$ is obtained by Eq.~\ref{eq_R}:
\begin{equation}
\label{eq_R}
R^{t+1}(j,k)= A^{t+1}(j,k) / (\boldsymbol{\sigma}^{t+1}_j \times \boldsymbol{\sigma}^{t+1}_k)
\end{equation}

With $R^{t+1}$, the following coefficient for each metric $j$ will be computed by Eq.~\ref{eq_Cj}:
\begin{equation}
\label{eq_Cj}
\boldsymbol{c}_j=\boldsymbol{\sigma}^{t+1}_j \cdot \sum_{k=1}^d(1-R^{t+1}(j,k))
\end{equation}

Given this coefficient vector, a new weight vector $\boldsymbol{\tilde{w}}$ is obtained according to Eq.~\ref{eq_wj}. 
\begin{equation}
\label{eq_wj}
\boldsymbol{\tilde{w}}_j=\frac{\boldsymbol{c}_j}{\sum_{k=1}^d \boldsymbol{c}_k}
\end{equation}

However, $\boldsymbol{\tilde{w}}$ is only the outcome depending on the current round. On the contrary, the aggregation step in Algo.~\ref{alg:flscorecard} uses Eq.~\ref{eq_w_expectation} to obtain the final weights.
\begin{equation}
\label{eq_w_expectation}
\boldsymbol{w}^{t+1} = (1-\frac{1}{t})\boldsymbol{w}^{t} + \frac{1}{t}\boldsymbol{\tilde{w}}
\end{equation}

It can be seen that this aggregation strategy of $Agg_w$ in Eq.~\ref{eq:fl_scoring} is a recursive aggregation method that considers historical weight changes, which can effectively reduce the impact of fluctuations in local parameters under statistical heterogeneity. In Section~\ref{sec:convergence}, the convergence analysis shows that as $t$ increases, the aggregated weights will approach the weights computed centrally on $X$.

In addition to the weight vector, our model also contains a vector of the metric-level scoring function $\vec f$. For a given value $x$ of the metric $j$, its scoring function $f_j(x)$ is formulated by Eq.~\ref{eq:metric_score_fn}. 
\begin{equation}
\label{eq:metric_score_fn}
f_{j}(x) =  \begin{cases}
F_j(x) & j \text{ is positive} \\
1-F_j(x))& j \text{ is negative}\\
1-2\times |F_j(\boldsymbol{\mu}_j) - F_j(x)| & j\text{ is oscillating} 
\end{cases}
\end{equation}
where $F_j$ is the CDF w.r.t metric $j$. If $j$ is a positive metric, then the CPD value $F_j(x)$ is the score; if the opposite, the score is 1-$F_j(x)$.
The third special case aims to show that $j$ is an oscillator and that typically $x$ follows a normal distribution. In this case, the mean tendency is preferred; therefore, the absolute value of $(F_j(\boldsymbol{\mu}_j) - F_j(x))$ is taken as the degree of deviation from the mean. The higher the deviation, the lower the score. Table~\ref{tab:CDf_forms} lists the PDFs and CDFs for the exponential and normal distributions used in this study. Note that our method is not limited to the two distributions but can cover more. This formulation not only considers the expectation types of metrics, but also normalizes each metric into the same range in case that metrics may have different scales. 

\begin{table}[htbp]\small
\caption{The two probability distributions used in this study}
\label{tab:CDf_forms}
\centering
\begin{tabular}{|c|c|c|}
\hline
\textbf{Distribution} & $x\sim E(\lambda)$ & $x\sim N(\mu, \sigma^2)$  \\ \hline
\textbf{PDF} & $\lambda e^{-\lambda x}$  & $\frac{1}{\sqrt{2\pi}\sigma} e^{-\frac{(x-\mu)^2}{2\sigma ^2} }$ \\ \hline 
\textbf{CDF} & $1-e^{-\lambda x}$ & $\frac{1}{\sqrt{2\pi}\sigma}\int_{-\infty }^{x}  e^{-\frac{(x-\mu)^2}{2\sigma ^2} }dx$ \\ \hline 
\textbf{Constraints} & $x\geq 0, \lambda=1/\mu>0$ & \\ \hline 
\end{tabular}
\end{table}

Finally, the federated scoring model $\hat{\mathcal{F}}$, depending on the globally estimated parameters: $\boldsymbol{w}^{T}$, $\boldsymbol{\mu}^T$ and $\boldsymbol{\sigma}^T$, could be achieved.

\subsubsection{Model inferring Process}
This section introduces how to score drivers after the federated scorecard model is built. First, the coordinator broadcasts the global model to all vehicles. Each vehicle then uses this model to evaluate its driving behavior in two steps as follows.

Given a member $\boldsymbol{x} \in X$, the first step is to calculate the vector of metric scores $\boldsymbol{q}$, where $\boldsymbol{q}_{j}=f_j(\boldsymbol{x}_j)$ for the $j^{th}$ metric. The second step is simply to summarize $\boldsymbol{q} \cdot \boldsymbol{w}^{\top}$ as the total score, which along with the corresponding driver ID will be returned to the coordination server. Throughout the whole process, the cloud server does not have access to the driver's private data but can still rank/distinguish drivers based on their summarizing scores.

\subsection{Convergence Analysis}
\label{sec:convergence}
This section provides a theoretical analysis of the convergence and consistency of the proposed method. For the metric statistics estimated by the CF4CRITIC-DM, we first present the following theorems. According to Theorem~\ref{them:unbiase_est_mean}, with increasing $t$, the global $\boldsymbol{\mu}$ and $\boldsymbol{\sigma}$ will eventually converge and their convergence direction is consistent with the actual mean and standard deviation of $X$.

\begin{them}
\label{them:unbiase_est_mean}
Assume $X$ a sample population in a homogeneous space $\mathbb{R}^d$, and its population mean and variance are $\mathbb{E}[X]$ and $\mathbb{D}[X]$ respectively . For each iteration $t$, it can be considered that $\mathcal{B}^t$ is the result of a random sampling of $X$, and $\boldsymbol{\mu}^t$ and ${(\boldsymbol{\sigma}^2)}^t$ are the mean and standard deviation estimated on $\mathcal{B}^t$. For a sufficiently small constant $\epsilon$, there exists 
\begin{equation}
\begin{aligned}
{\lim\limits_{t \to +\infty}}\mathcal{P}(|\boldsymbol{\mu}^t - \mathbb{E}[X]| > \epsilon)=0\\
{\lim\limits_{t \to +\infty}}\mathcal{P}(|\boldsymbol{(\sigma^2)}^t - \mathbb{D}[X]| > \epsilon)=0
\end{aligned}
\nonumber
\end{equation}
where $\boldsymbol{\mu}^t$ and $\mathbb{E}[X]$ are the estimated and real mean vector of $X$, respectively. ${(\boldsymbol{\sigma}^2)}^t$ and $\mathbb{D}[X]$ are the estimated and real variance of $X$, respectively.  
\end{them}

\begin{proof}
Our proof starts with the case where $K=1$, and the Theorem holds apparently since $X=X^{(1)}$;

When $\tau=1.0$, the theorem apparently holds since all clients participate in every training round, resulting in Eq.~\ref{eq:agg_sum} producing the same statistical items as those obtained with centralized computation;

When $K>1$ and $\tau \in (0,1)$, let $\mathcal{B}^t$, which contains local data from $K\times \tau$ clients, be the result of random sampling of $X$, then $\boldsymbol{\mu}^t$ is considered the mean vector of this sample, which is an unbiased estimate of $\mathbb{E}[X]$. Similarly, ${(\boldsymbol{\sigma}^2)}^t$ is an unbiased estimate of $\mathbb{D}[X]$.
\end{proof}

In addition, as $t$ increases, the aggregated weights obtained by our algorithm will also be close to the CRITIC weights calculated centrally on $X$. This conclusion is supported by the following Theorem~\ref{them:w_approximate}.

\begin{them}
\label{them:w_approximate}
Let $\boldsymbol{w}^*$ be the target weights computed on the invisible $X$ and let $\boldsymbol{w}^t$ be the aggregated global weight. If $t \to +\infty$, then there exists
\begin{equation}
{\lim\limits_{t \to +\infty}} \mathcal{P}(|\boldsymbol{w}^t - \boldsymbol{w}^*|> \epsilon)=0
\nonumber
\end{equation}
for a sufficiently small constant $\epsilon$.
\end{them}

\begin{proof}
Our proof starts with the case where $K=1$ or $\tau=1.0$, this theorem apparently holds since the Theorem~\ref{them:unbiase_est_mean} holds and then $\boldsymbol{w}^{t}=\boldsymbol{w}^{*}$;

When $K>1$ and $\tau \in (0,1)$, since $\boldsymbol{w}^*$ depends on $\mathbb{E}[X]$ and $\mathbb{D}[X]$, $\boldsymbol{\tilde{w}}$ depends on $\boldsymbol{\mu}^t$ and ${(\boldsymbol{\sigma}^2)}^t$, then $\boldsymbol{\tilde{w}}$ gradually approaches $\boldsymbol{w}^*$ according to Theorem 4.1, but there may be vibration behavior that affects its convergence. Therefore, Eq.~\ref{eq_w_expectation} is further used to obtain the global weight vector. 

Note that Eq.~\ref{eq_w_expectation} is actually the recursive form for computing the mean vector of the sample (deduced from Eq.~\ref{eq:recursive_mean}). This effectively alleviates the impact of vibration and divergence of the weight $\boldsymbol{w}^{t}$, which eventually converges to $\boldsymbol{w}^{*}$.
\end{proof}

\begin{equation}
\label{eq:recursive_mean}
\begin{aligned}
\boldsymbol{w}^t = &\frac{1}{t} \sum_{i=1}^t \boldsymbol{\tilde{w}}^i \\
\boldsymbol{w}^{t+1}=&\boldsymbol{w}^t+\frac{(\boldsymbol{\tilde{w}}-\boldsymbol{w}^t)}{t}\\
=&(1-\frac{1}{t})\boldsymbol{w}^{t} + \frac{1}{t}\boldsymbol{\tilde{w}}
\end{aligned}
\end{equation}

\subsection{Privacy Analysis}
To analyze the privacy protection ability, we first define a commonly used privacy threat model, where clients and the server honestly follow the secure protocols according to the promised procedures, but there exists an adversary who is willing to learn more to benefit itself.
According to \cite{paverd2014modelling}, the definition of adversary $\mathcal{A}$ in our privacy model is defined as follows:

\begin{definition}[Semi-honest Adversary]{
 $\mathcal{A}$ is a legitimate participant in a communication protocol who will not deviate from the defined protocol, but will attempt to learn all possible information from legitimately received messages. In FL, both the clients and the server can be the $\mathcal{A}$ with the following abilities: $\mathcal{A}$ can intercept communication channels and record transmitted messages; the cost of $\mathcal{A}$ to collide with or corrupt more than $m$ users is unaffordable; $\mathcal{A}$ cannot extract legitimate inputs or random seeds of other honest users; $\mathcal{A}$ has the polynomial-time computing power to launch attacks.
}\end{definition}

As mentioned, we adopt CRITIC-DM as the scoring model. The critical point is to find the descriptive statistics of the clients. If these statistics are uploaded directly in plaintext format, the central server has the chance to deduce the original data, such as the local extreme points. Given this definition, the realization of the vehicle's privacy requirements to the cloud server can be described as follows: Suppose that the coordination server is $\mathcal{A}$, and the driver ID and their trip ID can be shared to the coordination server, but the driving behavior of a specific driver/trip must not be inferred. 

To fulfill this requirement, the degree of privacy protection of our FedDriveScore solution is based on the following assumptions.

\begin{assum}
Assume that under the FedDriveScore framework proposed in Section~\ref{sec:sys_arch}, no participant can directly access the original data of other parties.
\end{assum}

\begin{assum}
The neutrality of the arbiter. The arbiter honestly complies with the FL protocols and cannot collude with coordination servers or edge vehicles to expose Paillier private keys.
\end{assum}

\begin{assum}
The anti-attack capability of the arbiter server. During the $T$ rounds of the federated training process, it is difficult for any participant to break through the arbiter during this period.
\end{assum}

\begin{assum}
\label{ass:dcra}
Decisional Composite Residuosity Assumption (DCRA), that is, given a composite number $n$ and an integer $r$, it is difficult to determine whether $r$ is an $n^{th}$ residual under $n^2$.
\end{assum}

Under Assumption~\ref{ass:dcra}, the following theorem is obtained according to \cite{paillier1999public}.
\begin{them}
Paillier encryption scheme meets semantic security, that is, ciphertext will not reveal any information of plaintext, and it is difficult to crack ciphertext in polynomial time.
\end{them}

Based on the assumptions and theorems described above, the method proposed in this paper ensures the security and privacy of vehicle data in the following aspects:

1) Vehicle-side data are kept locally and cannot be accessed directly. According to Assumption 4.1, the data of each participant is stored locally and not sent to other participants. This allows the vehicle to have full control over its own data;

2) In the parameter update process, there is a risk of inference attacks in the vehicle-side local model update. The parameters transmitted by this method include statistical information based on local data, so the coordination server can perform association speculation by continuously accessing local statistical items (such as samples and sums). Therefore, this paper further adopts the homomorphic encryption mechanism and introduces a trusted arbiter to manage the keys. Under the FedDriveScore architecture, the coordinator knows which clients participated in each round of FL and what they returned. However, the returned local statistical data are encrypted and cannot be decrypted by the coordinator. 

3) The decryption process is carried out on the arbiter side, but it cannot know which clients have participated in the training, nor can it know the private data returned by itself. In short, the arbiter does not have enough knowledge to infer the private information of individual clients from the decrypted intermediate results. For a client, it knows its own return content and can receive decrypted global parameters from the arbiter. However, the global parameters are aggregated by the server, so the client has no way to know the identities of other participants.

In summary, the proposed method is safe enough. The coordination server can only obtain the scoring model parameters, the driver or trip ID along with its summarizing score of each vehicle.

\section{Experiments and Results}
The following sections test and validate our method with real fleet driving data and virtual UBI data, respectively. The experiment is carried out on a server with an Intel i7-11700 CPU, 32GB memory, and the operating system is Ubuntu 18.04. All experiments take the CRITIC-DM parameters and the scoring results of the $\mathcal{F}$ obtained by CL as targets. We then simulate FL experiments to mainly prove the training convergence of the proposed method and the utility consistency of our federated scoring model with respect to the targets. It is worth mentioning that because the experimental data set has no prior score labels and the scoring model constructed in this unsupervised way cannot be verified like the clustering method, this section validates the utility loss of the federated scoring models based on the goal in the problem definition section.

For comparison, we also implement a CRITIC-DM model training method based on federated averaging. The idea is to directly perform recursive averaging of local CRITIC-DM models to obtain a global scoring model. In this method, the vehicle directly calculates the local mean, standard deviation, and weight of the metrics and returns them to the coordination server. The coordination server obtains the parameters corresponding to the metrics estimated in the current round through FedAvg aggregation. Then, with recursive update similar to Eq.~\ref{eq_w_expectation} to obtain the global mean, standard deviation, and weight of the metrics for the next round of training. The client selection rate of this method is fixed at 0.5, and the global scoring model trained in the same rounds as the proposed method is used for scoring.

\subsection{Fleet Driving Data}
In this section, a real fleet driving data~\cite{lu2021bi} is used to validate the practical effectiveness of our method for fleet management purposes. This data set collects driving trip data from the Telematics devices of 53 trucks in a logistics fleet. For each truck, its CAN-bus signals can be collected through the on-board Telematics Box (T-Box) and sent to the cloud server through the IoV communication channel.

With these data, the fleet manager hopes to evaluate and rank the driving performance of the hired drivers and use it as an objective measurement to take the corresponding rewards and punishments for excellent or poor drivers. However, most logistics fleets do not have the ability to independently construct a driver scoring model and need to rely on the data from the OEM and the strength of the algorithm service provider. If the traditional CL method is used to evaluate driving performance, it will involve the risk of leakage of vehicle sensor data and fleet operation information. Therefore, our FedDriveScore solution can effectively promote the collaboration of all parties to jointly build a scoring model.

Table~\ref{tab:truck_info} lists the information related to these data. Data collection time was January 1, 2019 to March 31, 2019, and the collection frequency was 1 Hz. A total of 32,135 trips are included, where the minimum, average and maximum trips per driver are 147, 606, and 995 trips, respectively. According to the fleet manager, after the delivery task was assigned, the assigned trucks were escorted between the surrounding cities. It should be noted that drivers are not given any instructions about the route they have to take or the speed or behavior they have to follow while driving, other than their destination. Therefore, the driver is free to drive in a completely uncontrolled open driving environment.

\begin{table}[htpb]\small
\caption{Truck configuration and data summary}
\label{tab:truck_info}
\begin{tabularx}{0.48\textwidth}{>{\setlength{\hsize}{0.4\hsize}}X|>{\setlength{\hsize}{0.6\hsize}}X}
\hline
\multicolumn{2}{c}{\textbf{Truck Configuration Information}}                        \\ \hline
truck brand                                       & Dongfeng Tianlong               \\
maximum speed                                     & 110$km/h$                       \\
maximum torque                                    & 2,000$N\cdot m$                 \\
maximum speed                                     & 4,000$rpm$                      \\
towing weight                                     & 40$tons$                        \\
idle range                                        & (0,700)$rpm$                    \\ \hline
\multicolumn{2}{c}{\textbf{Driving Trip Summary}}                                                              \\ \hline
number of vehicles                                & 53                                                         \\
road type                                         & Mainly intercity expressways with a speed limit of 90 km/h \\
\#trips                                           & 32,135                                                     \\
\#trips per vehicle                               & 606                                                        \\
shortest travel distance                          & 3km                                                        \\ \hline
\multicolumn{2}{c}{\textbf{Selected signals and their purpose}}                                                \\ \hline
engine torque ($N\cdot m$), speed ($rpm$)           & Detect ignition status or idle status                      \\
vehicle speed ($km/h$), brake pedal opening (\%) & Detect acceleration and braking events                     \\
bearing ($deg$)                                     & Detect turn events                                         \\
mileage ($km$)                                      & Calculate travel distance                                  \\
instantaneous fuel consumption ($mL$)               & Cumulative trip fuel consumption                           \\ \hline
\end{tabularx}
\end{table}

Since there are very few records containing missing values in the data table, such records are directly discarded in the experiment. In this experiment, two consecutive records are used to calculate the acceleration. First, the points whose absolute acceleration value is greater than 20 $km/h/s$ are outliers and will be ignored. Then, when the speed change rate reaches more than 3 $km/h/s$, events of rapid acceleration and rapid deceleration will occur. The effective bearing angle is between 0 and 360 degrees. When the vehicle speed exceeds 40 $km/h$ and the continuous bearing angle change rate is greater than 50 $deg/s$, it is considered a sharp turn event. The $99^{th}$ percentile of the instantaneous fuel consumption reading is used to filter the noisy fuel consumption reading. Also, the mileage signal has no outliers. The data set is taken from different routes, but with a larger proportion of highways to save time and avoid entering the city center. The data set is first pre-processed to handle outliers and missing records, and then split into trips by detecting if the engine is down for more than 30 minutes. Furthermore, data with consecutive missing values of more than 1 minute in the time series are also considered to belong to two different driving trips. Then, trips with a distance of less than 3 $km$ are discarded.

\subsubsection{Evaluation Metrics}
Under the assumption of FL, driving data are still kept in these 53 vehicles. For each driving trip, it is first necessary to define relevant evaluation metrics. Table~\ref{tab:truck_metrics} lists these evaluation metrics and their desired expectation types. Notice that the first 3 metrics focus on the aggressive driving behavior of the driver, which can be explained for good or bad performance. For example, sudden braking can be a signal that the driver is not paying attention or that a rear-end collision is likely. To obtain these metrics, the vehicle first uses predefined rules to detect these driving events in real time. The number of occurrences of these driving events is then recorded and divided by the distance traveled to obtain the frequency of the event per km. This normalization ensures that all trips are compared on the same scale.

The remaining 3 metrics in the table above assess a driver's skill at how to drive economically and efficiently. The idle ratio is the total idle time divided by the total run time for the trip. A high idle ratio will cause more fuel consumption and exhaust emissions, so this is a negative indicator that needs to be minimized. For the "AvgRPM" metric, since the most economical speed range is neither too high nor too low, it is regarded as an oscillatory indicator. This metric indicates that the greater the deviation from the average RPM or economy RPM, the worse the performance. "AvgSpd" is the metric to maximize. The higher the value, the longer it takes to reach the destination, which means a more convenient experience.

\begin{table}[]\small
\caption{Defined evaluation metrics for truck driving data}
\label{tab:truck_metrics}
\centering
\begin{tabularx}{0.49\textwidth}{>{\setlength{\hsize}{0.8\hsize}}X|>{\setlength{\hsize}{0.2\hsize}}X}
\hline
\multicolumn{1}{c|}{\textbf{Driving trip evaluation metrics}} & \multicolumn{1}{c}{\textbf{Expectation type}} \\ \hline
Harsh Accelerations per Kilometer (\#HarshAcc)              & negative                                                  \\
Hard braking times per kilometer (\#HarshDec)               & negative                                                  \\
Number of sharp turns per kilometer (\#HarshTurn)           & negative                                                  \\
Parking idle ratio (IdleRatio)                              & negative                                                  \\
Average Speed (AvgSpd)                                      & positive                                                  \\
Average Engine Speed (AvgRPM)                               & oscillator                                     \\ \hline
\end{tabularx}
\end{table}

Next, secure histogram plotting is performed to observe the distributions of evaluation metrics. The final plots are shown in Fig.~\ref{hosf_fleet}, the first 3 metrics are all exponentially distributed. On the contrary, the histograms for the remaining metrics show normal distributions. Based on the above observations, we can determine the distribution type of each metric. For comparison, Fig.~\ref{fig:client_fleet_hosf} shows the distribution of evaluation metrics in the local trip data of a certain truck. Although the distribution types of the vehicle's local indicator data are similar, the statistical items show that the local distribution and the overall distribution are inconsistent, that is, the non-identical distribution.

\begin{figure}[!htbp]
\centering
\subfigure[] {
      \includegraphics*[width=1.6in]{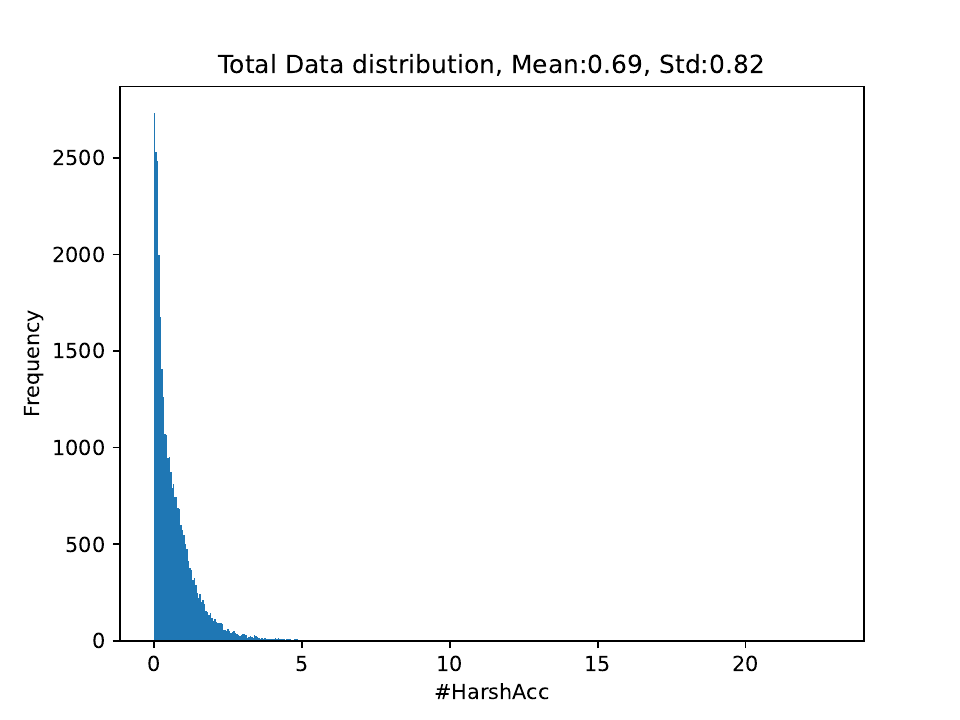}
  }
\subfigure[] {
      \includegraphics*[width=1.6in]{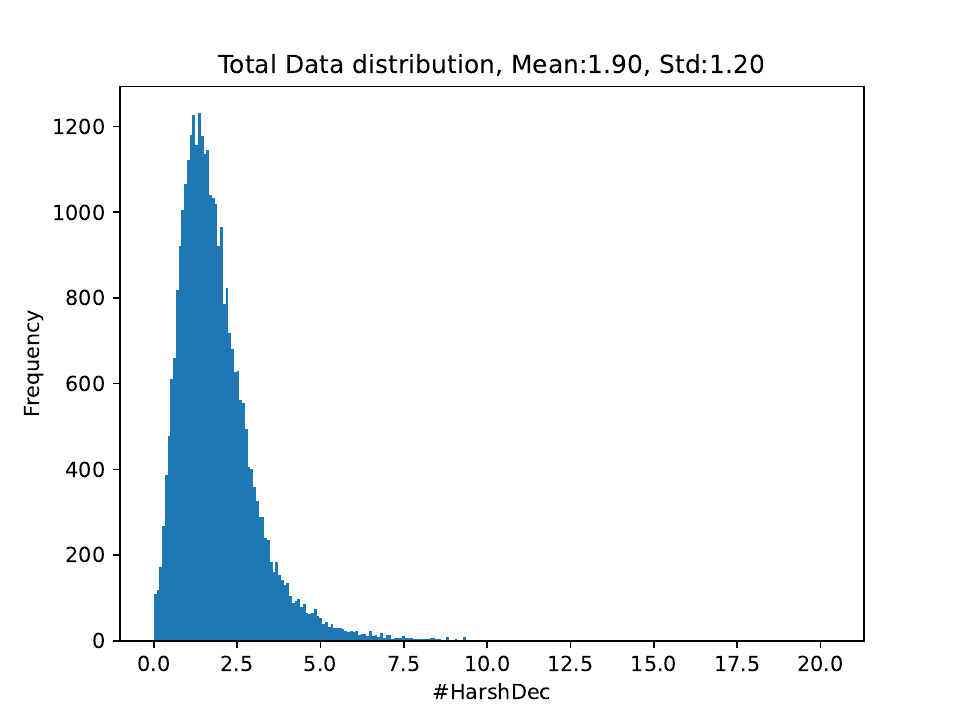}
  }
\subfigure[] {
      \includegraphics*[width=1.6in]{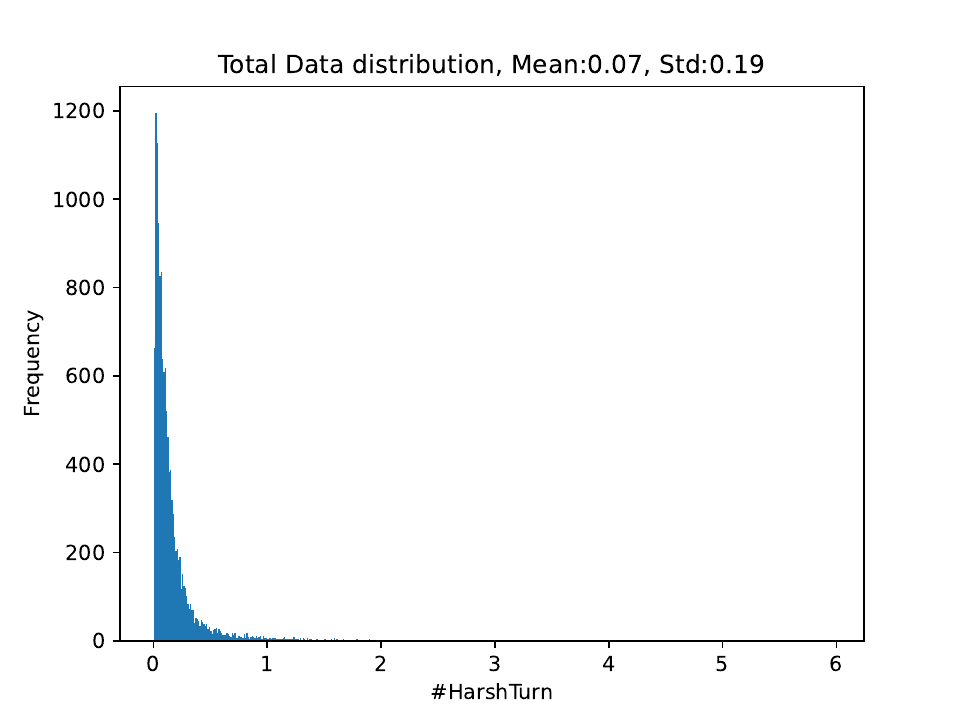}
}
\subfigure[] {
      \includegraphics*[width=1.6in]{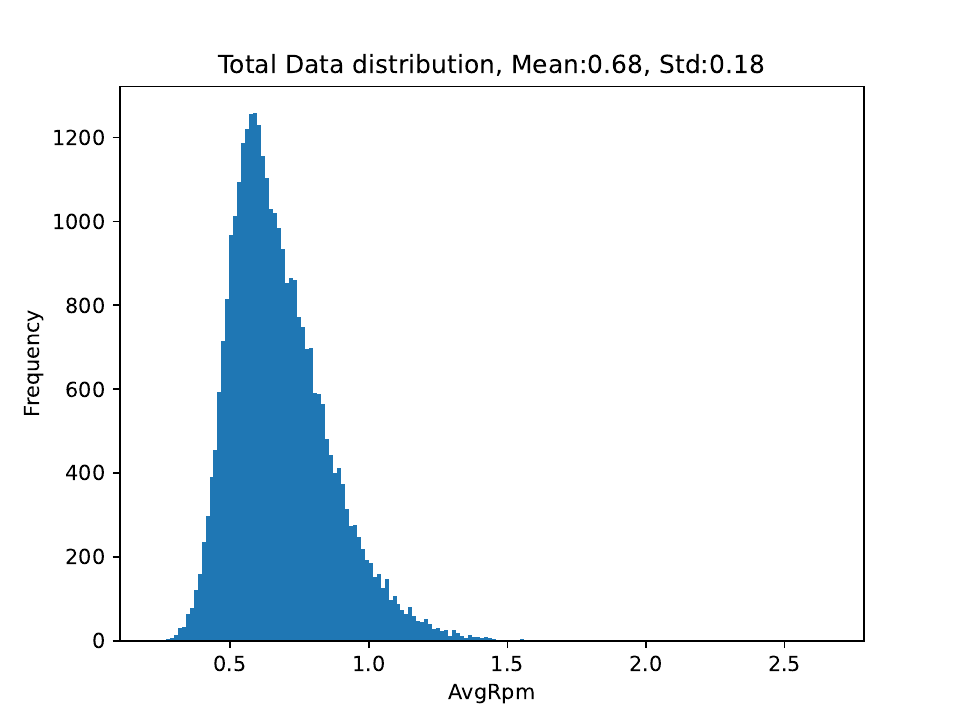}
  }
\subfigure[] {
      \includegraphics*[width=1.6in]{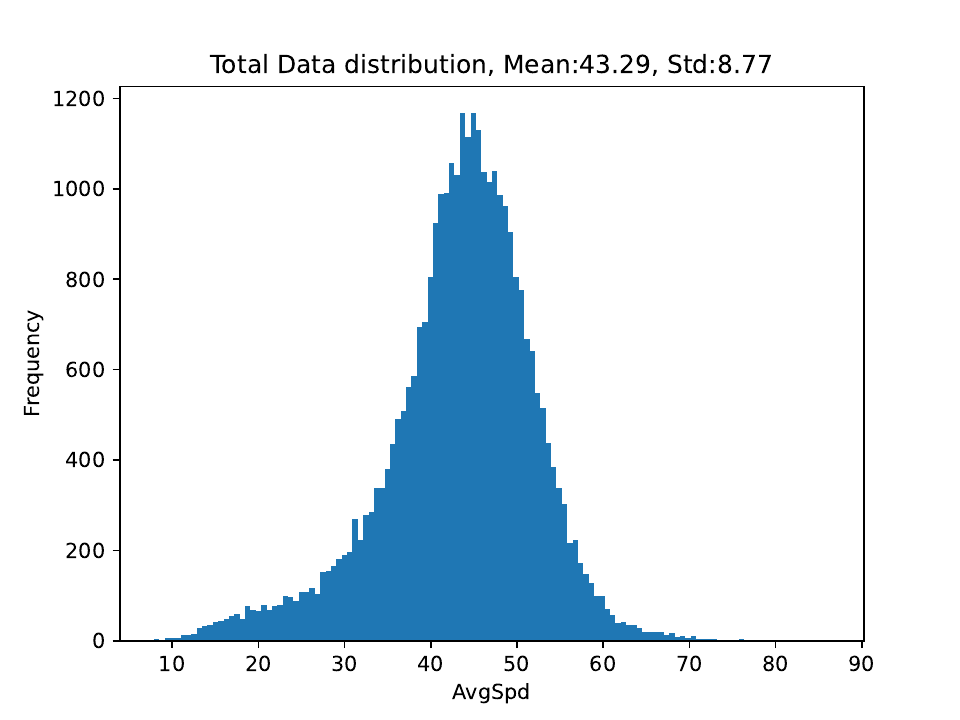}
  }
\subfigure[] {
      \includegraphics*[width=1.6in]{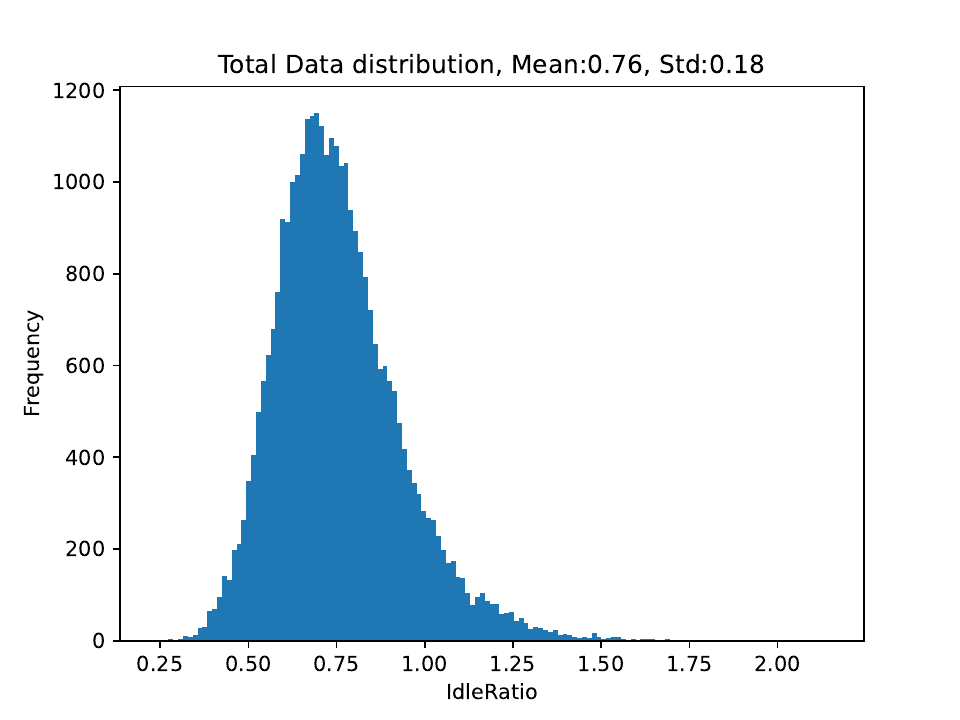}
  }
\caption{The overall distributions of metrics in the fleet driving data.}
\label{hosf_fleet}
\end{figure}

\begin{figure}[!htbp]
\centering
\subfigure[] {
      \includegraphics*[width=1.6in]{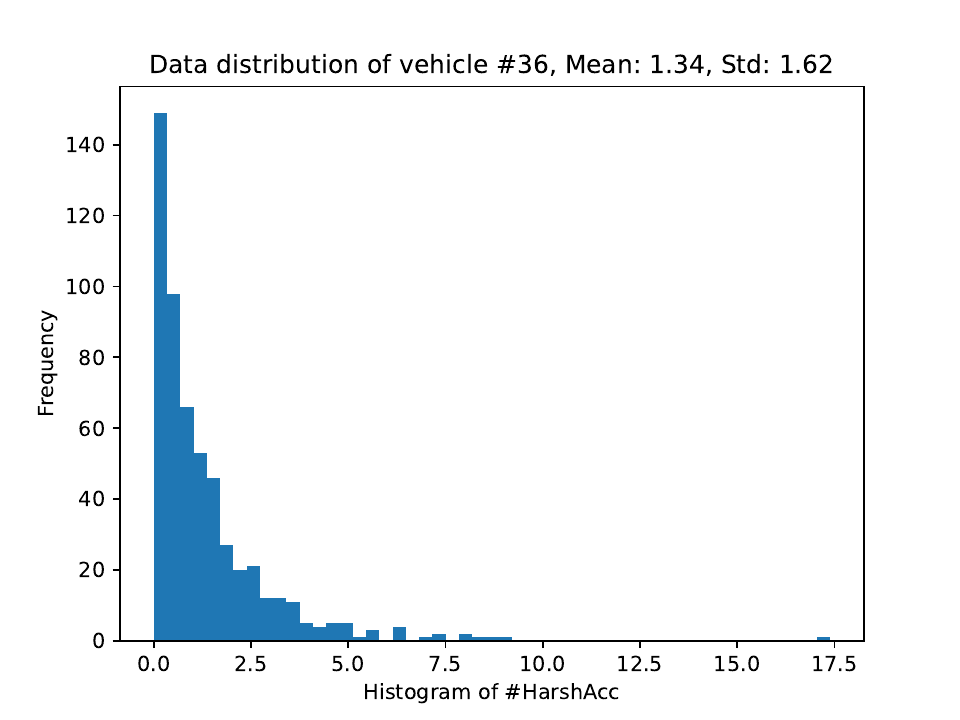}
  }
\subfigure[] {
      \includegraphics*[width=1.6in]{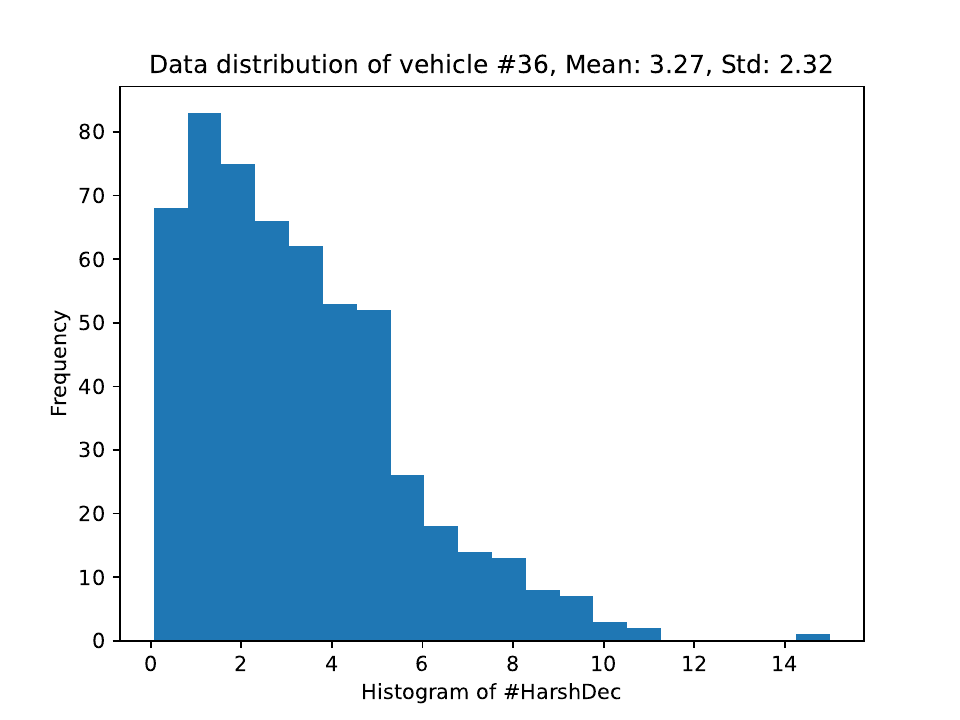}
  }
\subfigure[] {
      \includegraphics*[width=1.6in]{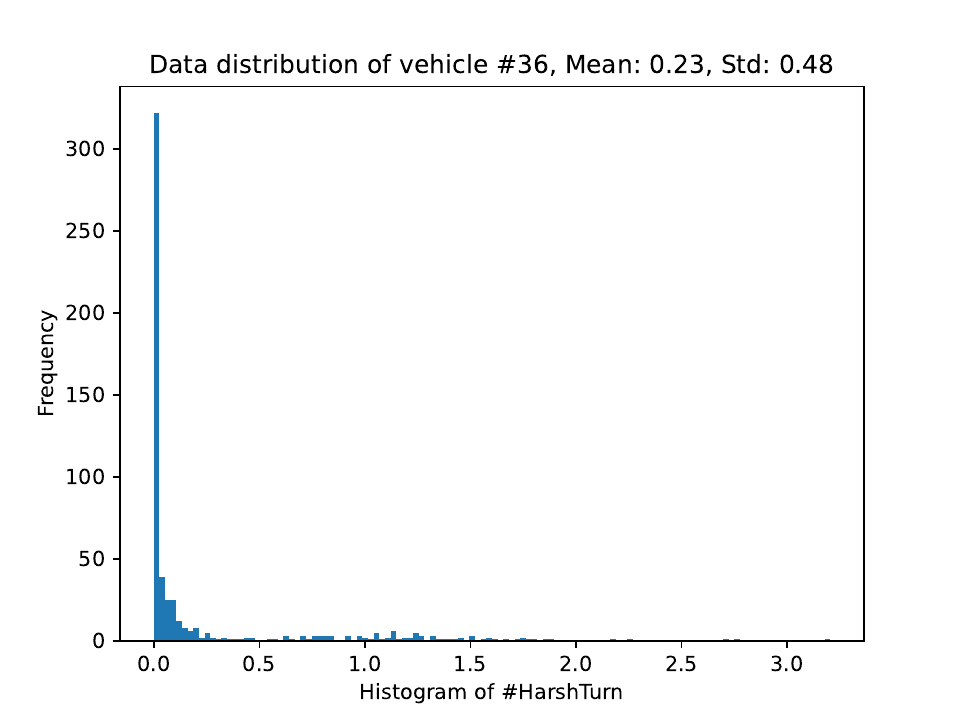}
}
\subfigure[] {
      \includegraphics*[width=1.6in]{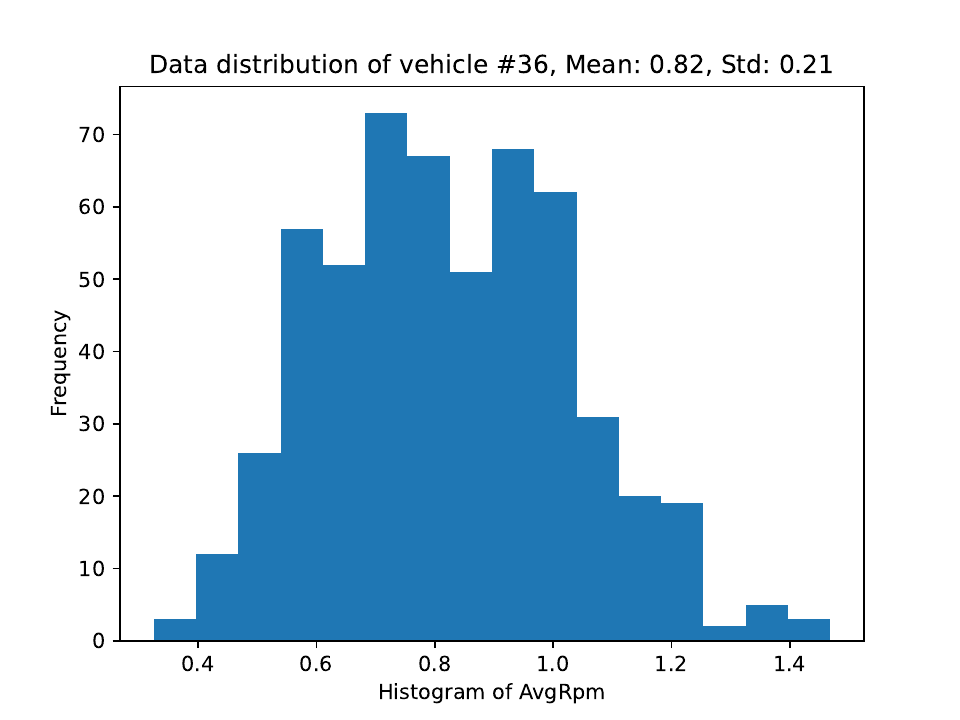}
  }
\subfigure[] {
      \includegraphics*[width=1.6in]{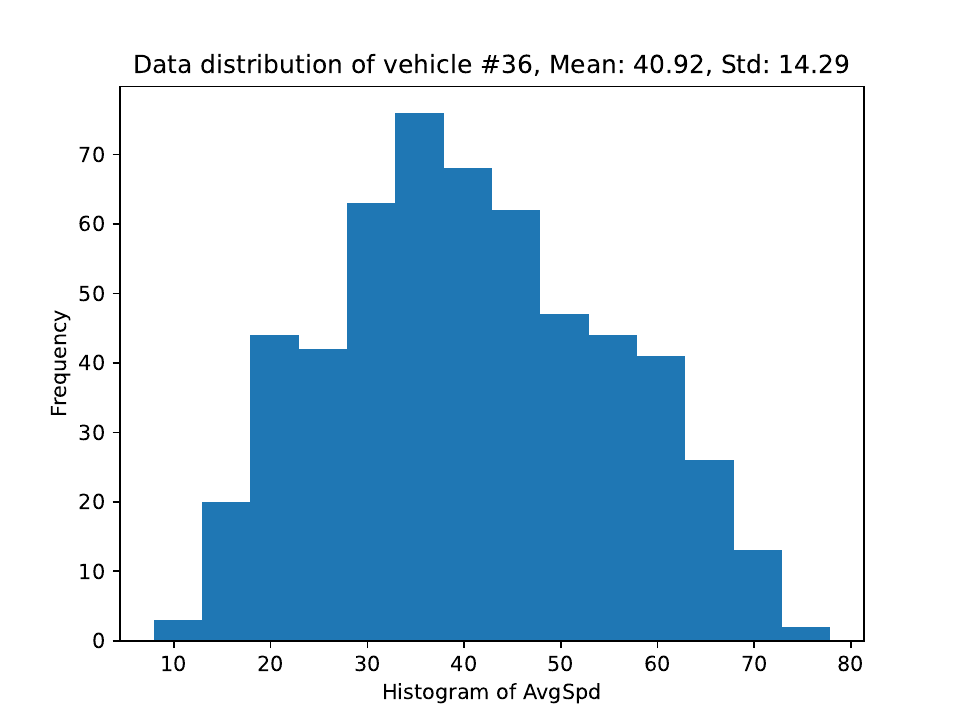}
  }
\subfigure[] {
      \includegraphics*[width=1.6in]{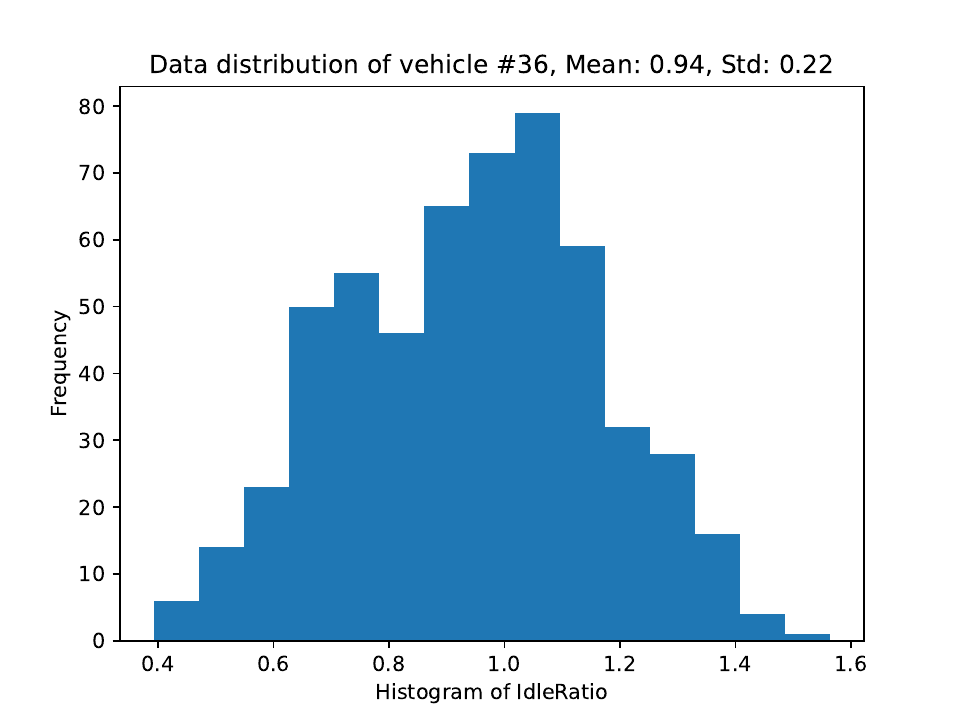}
  }
\caption{A client-level histograms of metrics in the fleet driving data.}
\label{fig:client_fleet_hosf}
\end{figure}

\subsubsection{Convergence Validation}
In the model training phase, we set $T$ to 300 and study the changing trend of estimated statistics and weights in each round when the client selection rate is 0.1 and 0.5.

Under the above two client selection rates, the convergence trend of metrics' means and variances is shown in Fig.~\ref{fig:fleet_mean_var}. Among the four subfigures, the dotted line represents the values vs. rounds estimated by CF4CRITIC-DM, and the solid line of the corresponding color is the target line given by CRITIC-DM. The metric mean vector on this data set actually includes 6 metrics, but there are only 4 lines that are clearly visible due to some lines overlapping. It can be observed that as the number of rounds increases, all dotted lines gradually converge and move closer to their target lines after experiencing short-term fluctuations, no matter when the selection rate is 0.1 or 0.5.

\begin{figure*}[!htbp]
\centering
\subfigure[$\tau=0.1$] {
\label{fig:fleet_mean_vara}
\includegraphics*[width=1.6in]{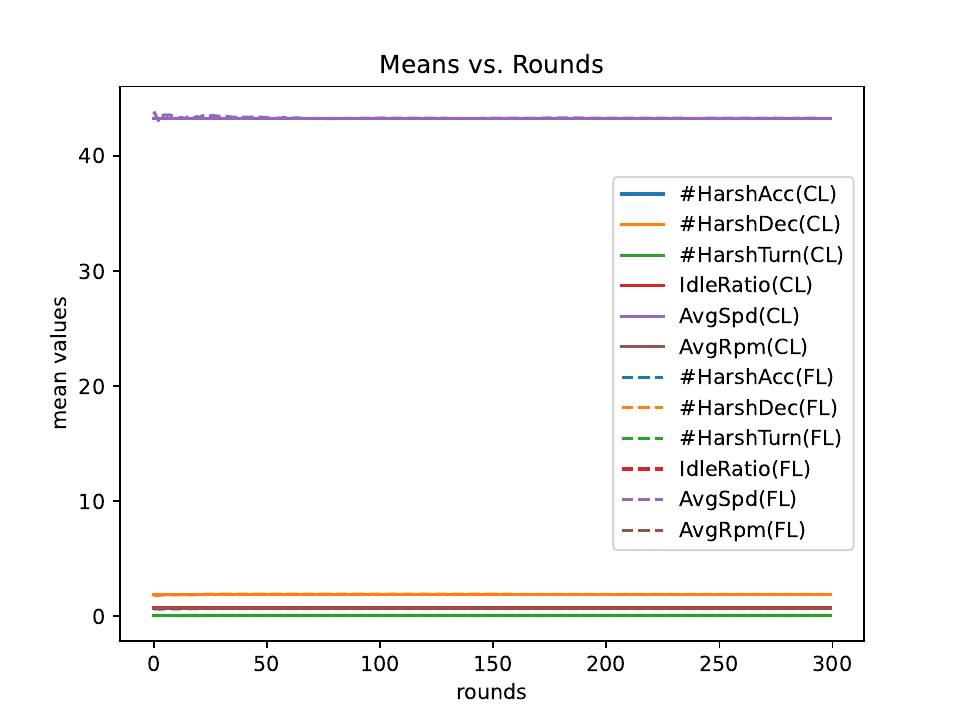}
}
\subfigure[$\tau=0.5$] {
\label{fig:fleet_mean_varb}
      \includegraphics*[width=1.6in]{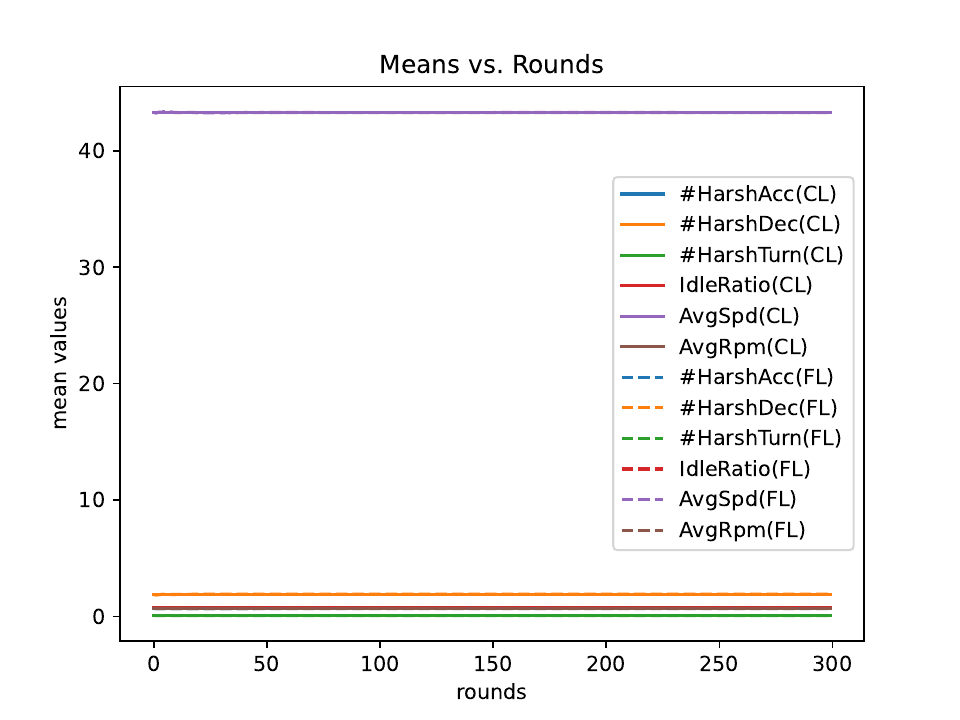}
  }
\subfigure[$\tau=0.1$] {
\label{fig:fleet_mean_varc}
      \includegraphics*[width=1.6in]{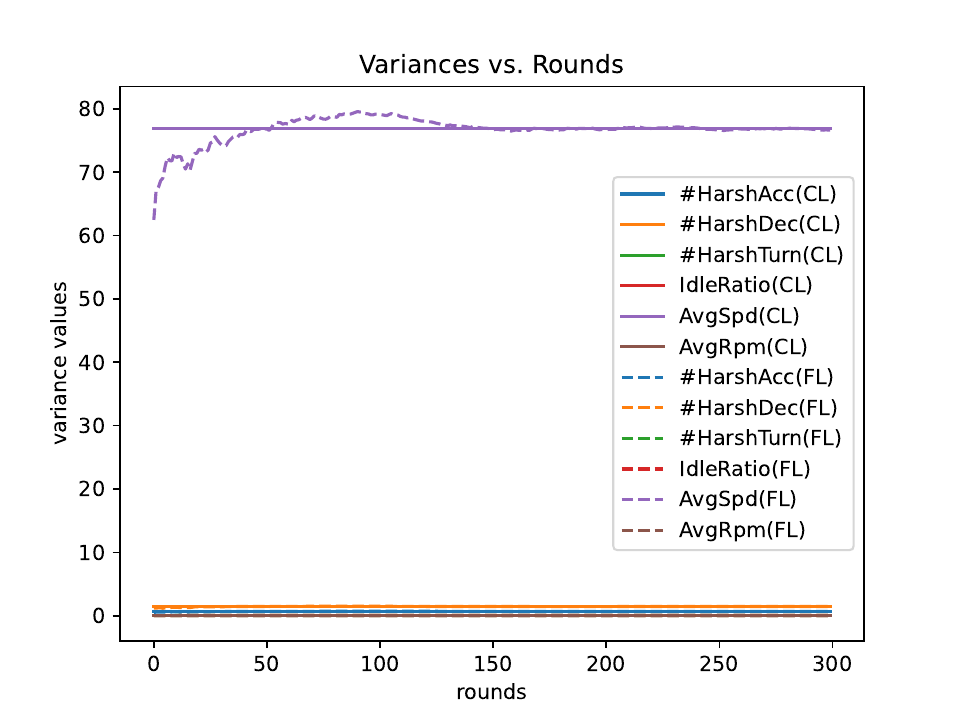}
}
\subfigure[$\tau=0.5$] {
\label{fig:fleet_mean_vard}
      \includegraphics*[width=1.6in]{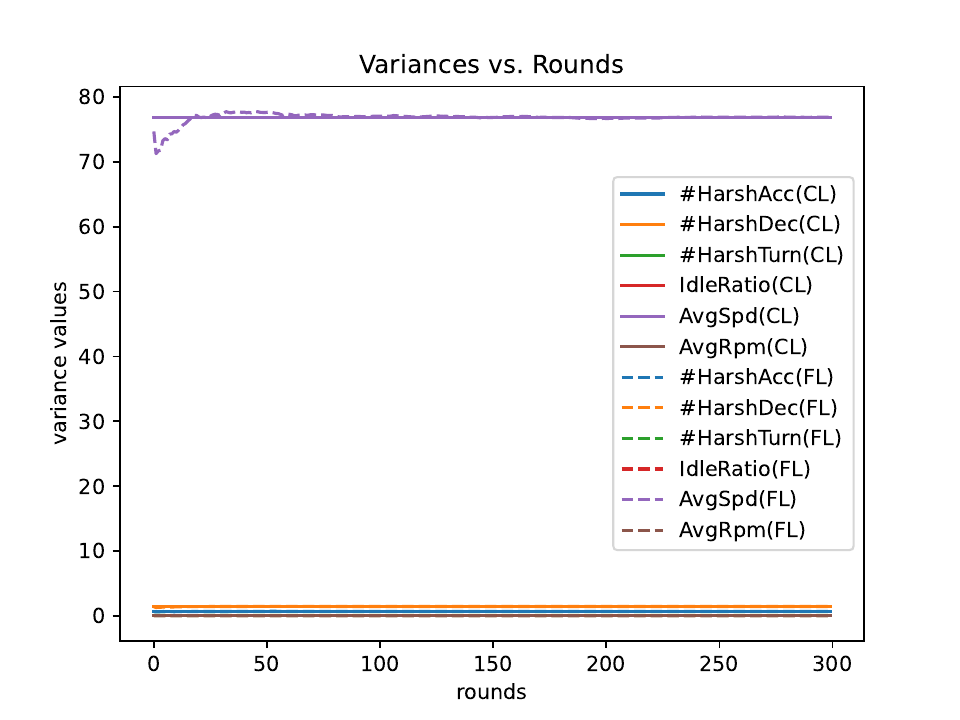}
  }
\caption{The estimated mean and variance of the metrics in the fleet driving data.}
\label{fig:fleet_mean_var}
\end{figure*}

Next, we analyze the convergence of the global metric weights estimated by our method. Fig.~\ref{fig:fleet_weights} plots the estimated weights of the 6 metrics that change with respect to the training rounds. It can also be observed from the dotted lines in these two figures that the estimated weights oscillate greatly at the beginning and deviate from their target weights. However, as the number of training rounds increases, they will eventually converge to their expected levels.

\begin{figure}[!htbp]
\centering
\subfigure[$\tau=0.1$] {
      \includegraphics*[width=1.6in]{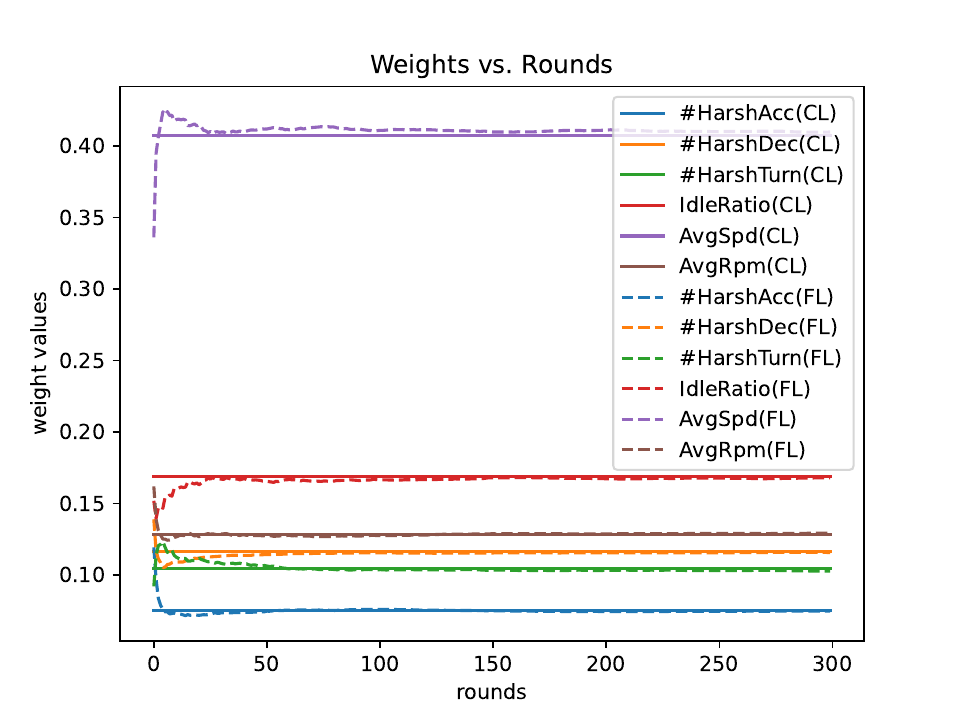}
}
\subfigure[$\tau=0.5$] {
      \includegraphics*[width=1.6in]{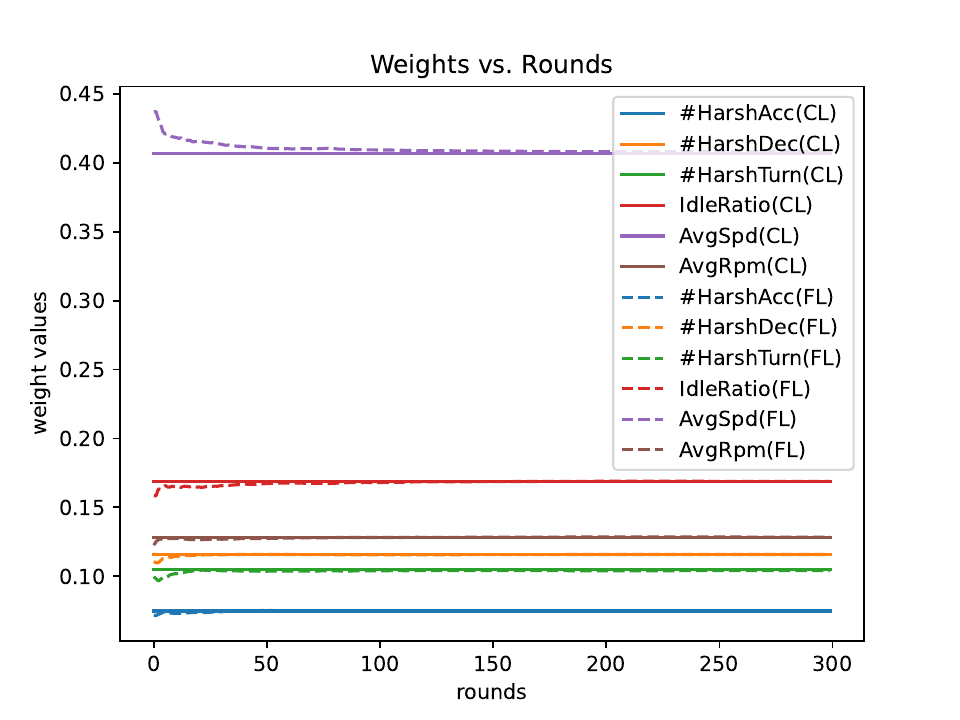}
}
\caption{The estimated weights of metrics in the fleet driving data.}
\label{fig:fleet_weights}
\end{figure}

In addition, from these three experimental results, we can also see the impact of $\tau$ on the convergence speed: When it is 0.5, the weights of the federated model converge faster compared to when it is 0.1, and so are the estimated metric statistics.

\subsubsection{Utility Consistency Validation}
After finishing the training process, two federated scoring models (denoted $\hat{\mathcal{F}}_{|\tau=0.1}$ and $\hat{\mathcal{F}}_{|\tau=0.5}$) could be obtained. We then analyze the utility consistency of both models compared to the CL scoring model (denoted $\mathcal{F}$) in this section.

First, in terms of the model parameters, the CRITIC weight vector in $\mathcal{F}$ and those in $\hat{\mathcal{F}}_{|\tau=0.1}$ and $\hat{\mathcal{F}}_{|\tau=0.5}$ are listed in Table~:

\begin{table*}[]\small
\caption{Comparison of metric weight from different scoring models on fleet driving data}
\label{tab:fleet_metric_weight}
\centering
\begin{tabular}{l|c|c|c|c|c|c}
\hline
Models & $\#HarshAcc$ & $\#HarshDec$ & $\#HarshTurn$ & IdleRatio & AvgSpd &AvgRPM \\ \hline
$\mathcal{F}$&0.07478049 & 0.11605587 &0.10467963 &0.16904160 & 0.40712076 &0.12832165 \\
$\hat{\mathcal{F}}_{|\tau=0.1}$ & 0.07468682 & 0.11570110 & 0.10269152&0.16784807& 0.40987586&0.12919664 \\
$\hat{\mathcal{F}}_{|\tau=0.5}$& 0.07481089 &0.11577172 &0.10409142&0.16880511&0.40715536&0.12836550 \\ \hline
\end{tabular}
\end{table*}

By intuitive comparison, we can see that the metric weights of the 3 models are very close. Especially when the client selection rate is 0.5, the weights of $\hat{\mathcal{F}}_{|\tau=0.5}$ are confident enough to represent the weights of each metric in $\mathcal{F}$.

\begin{table*}[htbp]\small
\caption{Pearson coorelation coefficients between metrics of fleet driving data}
\label{tab:pearson_fleet_metrics}
\centering
\begin{tabular}{|c|c|c|c|c|c|c|}
\hline
\textbf{} & HarshAcc       & HarshDec       & Harsh Turn & IdleRatio      & AvgSpd & AvgRPM         \\ \hline
HarshAcc  & 1.000          & \textbf{0.773} & 0.037      & 0.604          & -0.166 & \textbf{0.809} \\ \hline
HarshDec  & \textbf{0.773} & 1.000          & 0.043      & \textbf{0.811} & -0.187 & \textbf{0.794} \\ \hline
HarshTurn & 0.037          & 0.043          & 1.000      & 0.074          & 0.057  & 0.069          \\ \hline
IdleRatio & 0.604          & \textbf{0.811} & 0.074      & 1.000          & 0.134  & \textbf{0.829} \\ \hline
AvgSpd    & -0.166         & -0.187         & 0.057      & 0.134          & 1.000  & 0.014          \\ \hline
AvgRPM    & \textbf{0.809} & \textbf{0.794} & 0.069      & \textbf{0.829} & 0.014  & 1.000          \\ \hline
\end{tabular}
\end{table*}

Furthermore, the numerical values in the weight vector show that "AvgSpd" and "IdleRatio" are given greater importance, which is related to convenient and economical driving. On the contrary, "$\#$HarshAcc" is given the lowest weight. As can be observed from the matrix of metric correlation coefficients calculated in Table~\ref{tab:pearson_fleet_metrics}, "$\#$HarshAcc" has a high direct correlation with "$\#$HarshDec" and "AvgRPM", but "AvgSpd" has a low correlation with other metrics. This shows the better weighting strategy of the CRITIC method by considering inter-correlation of different metrics.

Next, we use our federated models to score the data separately and observe their score distributions. The distributions are then compared with the one given by $\mathcal{F}$. Fig.~\ref{fig:fleet_score_compare} shows the histograms of the scores scored by the CL and FL models, respectively. It can be seen that all of these subgraphs exhibit a normal distribution trend, that is, most trip scores are concentrated in the middle area and trips with higher or lower scores are in the minority. This kind of distribution is beneficial to the fleet manager's perception of the overall level of the driver group. In addition, the score distributions plotted in Fig.~\ref{fig:fleet_score_compareb} and Fig.~\ref{fig:fleet_score_comparec} are with very slightly average score difference, but the same stand deviation as in Fig.~\ref{fig:fleet_score_comparea}.

\begin{figure}[!htbp]
\centering
\subfigure[$\mathcal{F}$] {
\includegraphics*[width=1.0in]{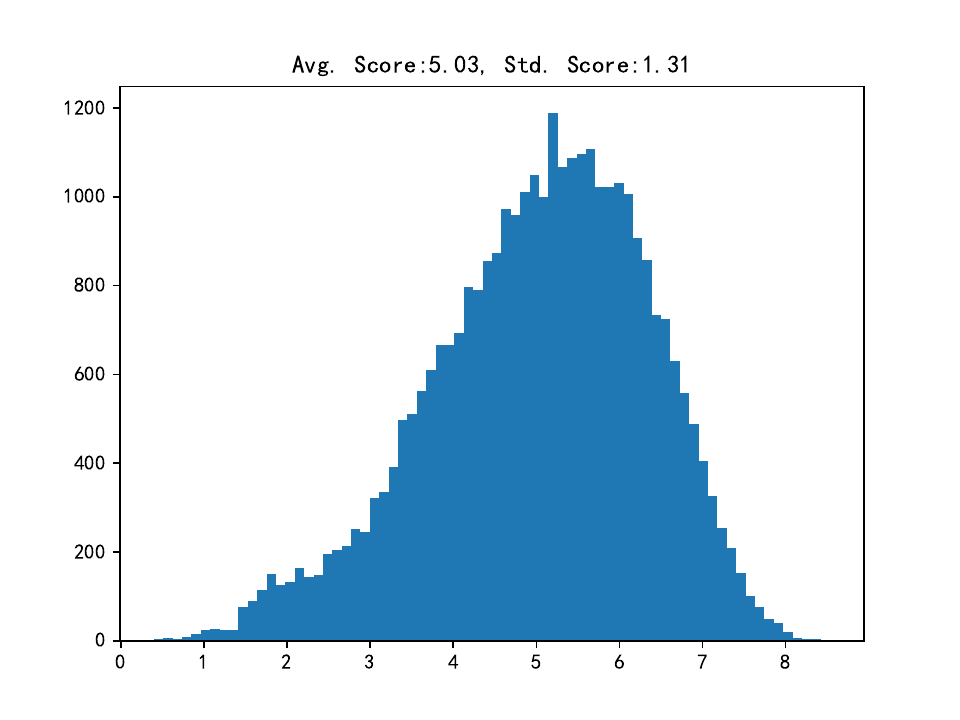}\label{fig:fleet_score_comparea}
}
\subfigure[$\hat{\mathcal{F}}_{|\tau=0.1}$] {
\includegraphics*[width=1.0in]{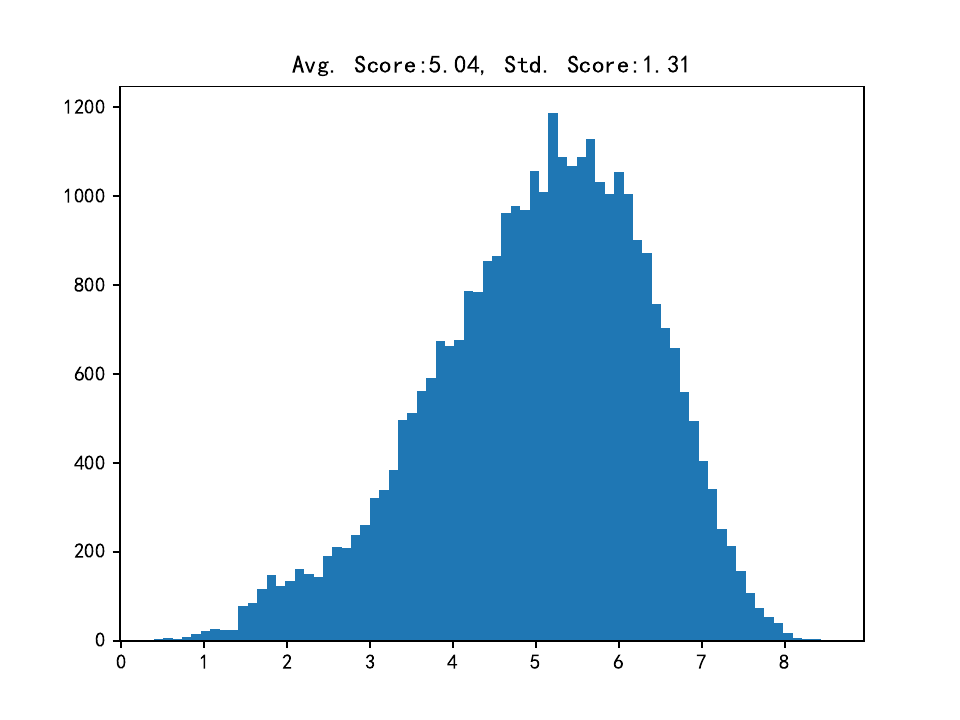}\label{fig:fleet_score_compareb}
  }
\subfigure[$\hat{\mathcal{F}}_{|\tau=0.5}$] {
\includegraphics*[width=1.0in]{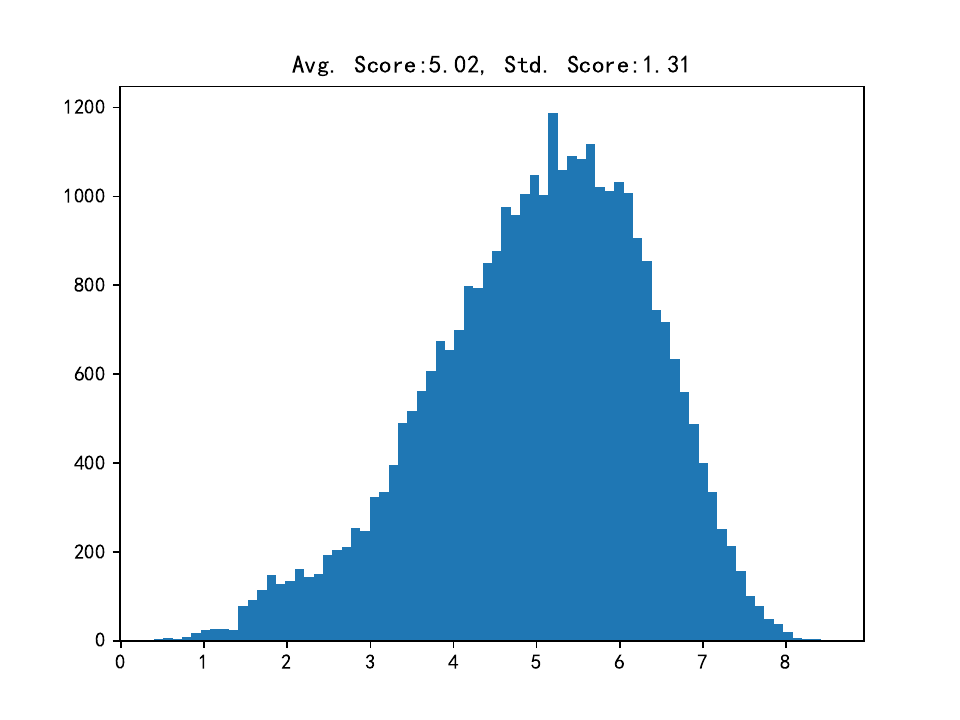}\label{fig:fleet_score_comparec}
}
\caption{Comparison of trip score distributions on fleet driving data.}\label{fig:fleet_score_compare}
\end{figure}

Furthermore, with the scores given by $\mathcal{F}$ as target labels, we use the common regression evaluation indexes, namely MSE, MAE, RMSE, and $R^2$-score, to verify the effect of our federated scoring models. Table~\ref{tab:fleet_score_val} lists these index values. In this table,  the FedAvg method is considered for comparison. We denote it as $\text{FedAvg}_{|\tau=0.5}$, which is trained with $\tau=0.5$ and the same number of training rounds as ours.

Our purpose is to show that the performance of FedAvg is affected by the skewness of the metric distribution. As expected, there is a large utility loss between $\text{FedAvg}_{|\tau=0.5}$ and $\mathcal{F}$. On the contrary, the difference in the scores of our two federated models is very small. The utility of our models is almost the same as $\mathcal{F}$. 
The above results show that once the training process converges, it is quite possible to obtain a lossless FL scoring model through the proposed CF4CRITIC-DM method.

\begin{table}[htbp]\small
\caption{Model score validation indexes on fleet driving trips}
\label{tab:fleet_score_val}
\centering
\begin{tabular}{c|l|l|l|l}
\hline
\textbf{} & $\mathcal{F}$ & $\hat{\mathcal{F}}_{|\tau=0.1}$ & $\hat{\mathcal{F}}_{|\tau=0.5}$ & $\text{FedAvg}_{|\tau=0.5}$ \\ \hline
MSE       & 0                      & 0.0001                & 0.0000   & 0.1322             \\ 
MAE       & 0                      & 0.0088                & 0.0011  & 0.3005              \\ 
RMSE      & 0                      & 0.0066            & 0.0035 & 0.3635          \\ 
$R^2$-score        & 1.0                    & 0.9999            & 1.0000    & 0.9419       \\ \hline
\end{tabular}
\end{table}

\subsection{Virtual UBI Data}
In addition to the fleet management application, evaluating driving performance is also critical in the development of the UBI application. Insurance companies can formulate personalized rate adjustment strategies based on the driving performance scores of different drivers. Therefore, this section selects the virtual UBI score data used by \citet{handel2014insurance} and \citet{lopez2018genetic} for further testing and verification purposes.

The data consist of 200 virtual driving trips, each of which includes 8 driving events, and the generated values can span a wider range of driving scores. In particular, \cite{lopez2018genetic} also invited 50 different subjects to rate these trips individually to assess the risk of driving. The records of each trip were scored by three subjects and then the average value was taken as the final subjective scoring result.

Table~\ref{tab:one_ubi_trip} lists an example of a given driving trip. In this table, the first 2 (distance and average speed) are each normalized to a 1-10 range from shortest/slowest to longest/fastest travel. The following six features represent the occurrence times of different risky driving maneuvers. Those who participated in the scoring were asked to assign an integer number from 1 to 10 to the entire trip based on their perception of the safety/risk of driving, with 1 representing very dangerous or unsafe travel and 10 representing very safe travel.

\begin{table}[htbp]
\caption{Example of a virtual trip and a score given by a subject}
\label{tab:one_ubi_trip}
\centering
\scalebox{0.8}{
\begin{tabular}{|l|c|c|}
\hline
\textbf{driving events}                        & \textbf{value} & \textbf{subjective score} \\ \hline
Distance                                       & 7              & \multirow{8}{*}{8}          \\ \cline{1-2}
Average speed                                  & 6              &                             \\ \cline{1-2}
\# of acceleration events(\#Accel)             & 5              &                             \\ \cline{1-2}
\# of sudden starts(\#SuddenStart)             & 3              &                             \\ \cline{1-2}
\# of abrupt lane changes(\#AbruptLanceChange) & 2              &                             \\ \cline{1-2}
\# of intense brakes(\#IntenseBrake)           & 7              &                             \\ \cline{1-2}
\# of sudden stops(\#SuddenStop)               & 0              &                             \\ \cline{1-2}
\# of abrupt steerings(\#AbruptSteering)       & 1              &                             \\ \hline
\end{tabular}
}
\end{table}

Since there are no driver IDs associated with these trips, we randomly distributed the 200 trips to 40 clients with unbalanced local distributions to simulate the FL setting. 
We simulate the skewness of the metric distribution based on the shards partition strategy provided by the fedlab~\cite{JMLR:v24:22-0440} tool, where the number of shards is 40.
Furthermore, the number of training rounds is set to 1000, and the client selection rate is also 0.1 and 0.5. Then we mainly verify the effectiveness of the proposed method from two points of view. The first point is to compare the federated scoring model with the CL scoring model to prove the consistency of the two. The other point is to compare the scoring results of our scoring model with the mentioned subjective scores to further illustrate the fairness and objectivity of our approach.

\subsubsection{Evaluation Metrics}
For this data, the distance was first used to normalize other events and convert them into evaluation metrics, such as the average speed per km and the number of sudden brakes per km. Such a normalization allows all trips to be compared on the same scale. Next, through the histogram of these metrics shown in Fig.~\ref{fig:hosf_ubi}, it can be seen that all metrics obey the exponential distribution. Also, except for the average speed per kilometer (AvgSpd), which is an oscillating indicator, all the others are negative metrics.

Fig.~\ref{fig:client_hosf_ubi} shows an example of client-level metric distributions in the virtual UBI data. As can be seen, there is a significant difference in the local histogram compared to those in Fig.~\ref{fig:hosf_ubi}. We will then validate whether our method can overcome such a statistical heterogeneity.   

\begin{figure*}[htbp]
\centering
\subfigure[] {
      \includegraphics*[width=1.6in]{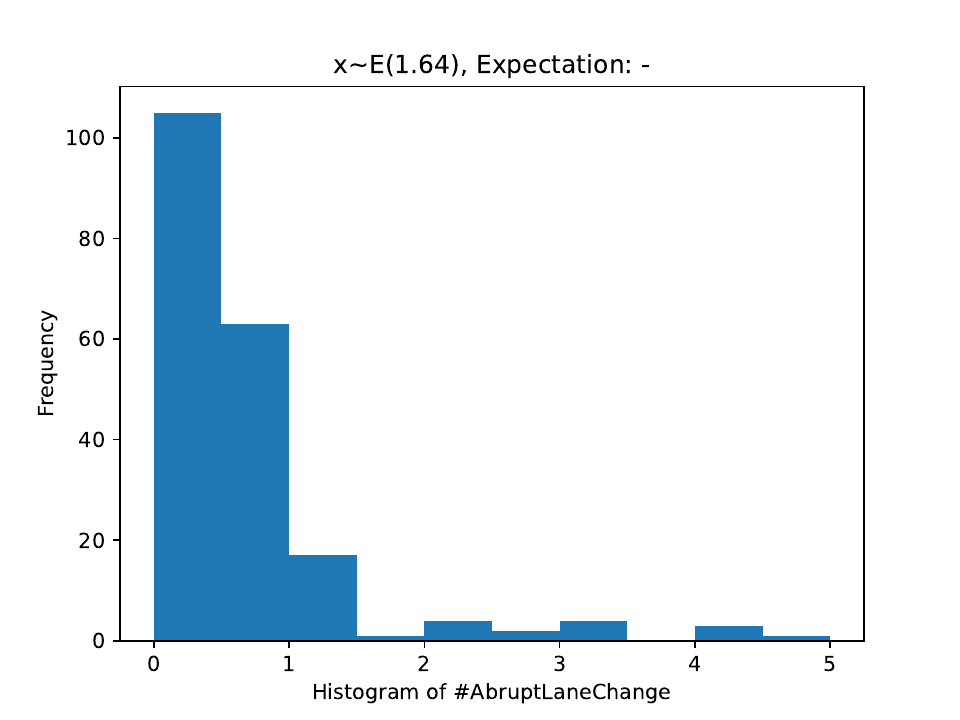}
  }
\subfigure[] {
      \includegraphics*[width=1.6in]{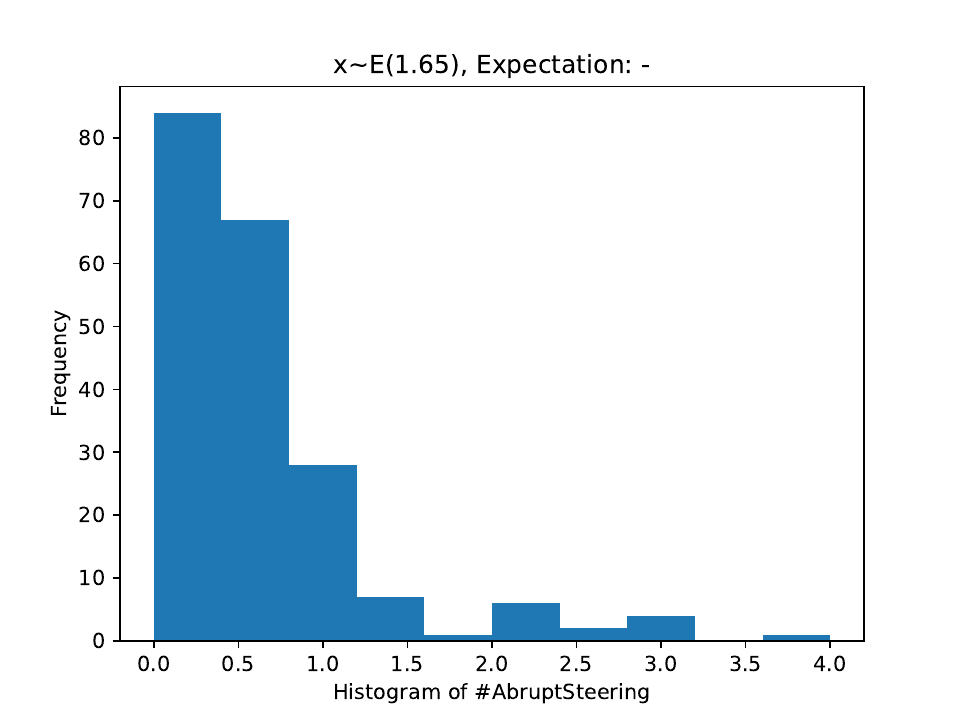}
  }
\subfigure[] {
      \includegraphics*[width=1.6in]{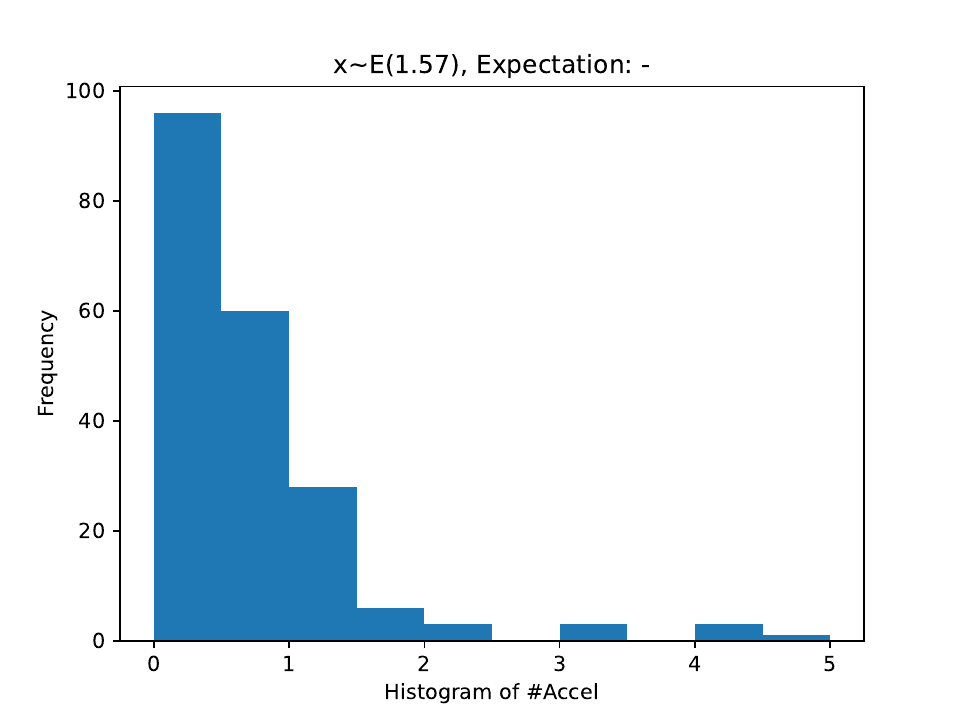}
}
\subfigure[] {
      \includegraphics*[width=1.6in]{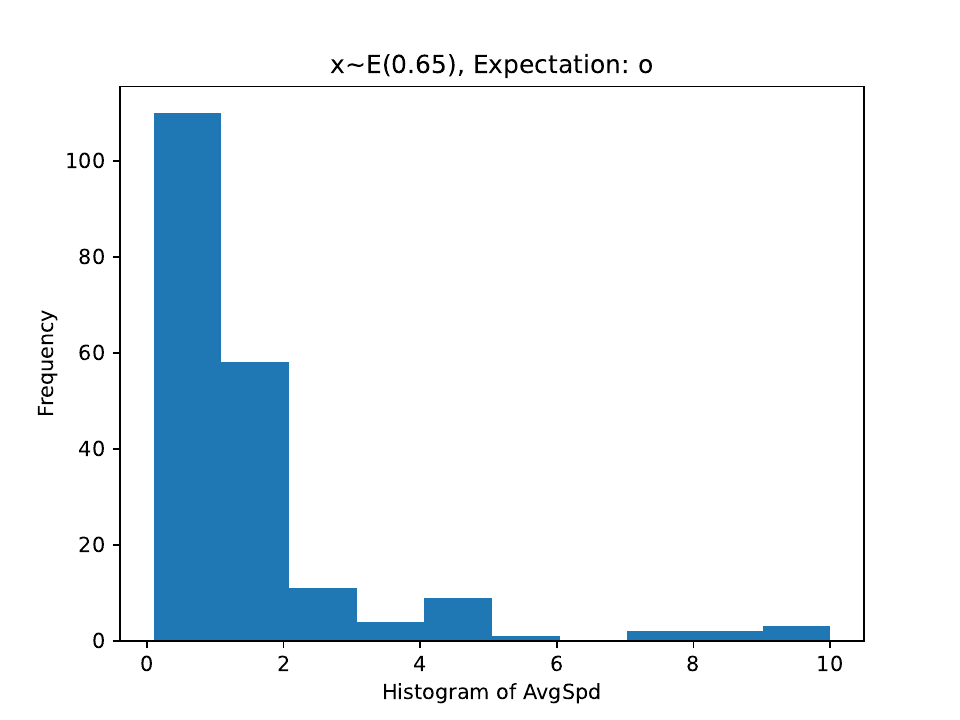}
  }
\subfigure[] {
      \includegraphics*[width=1.6in]{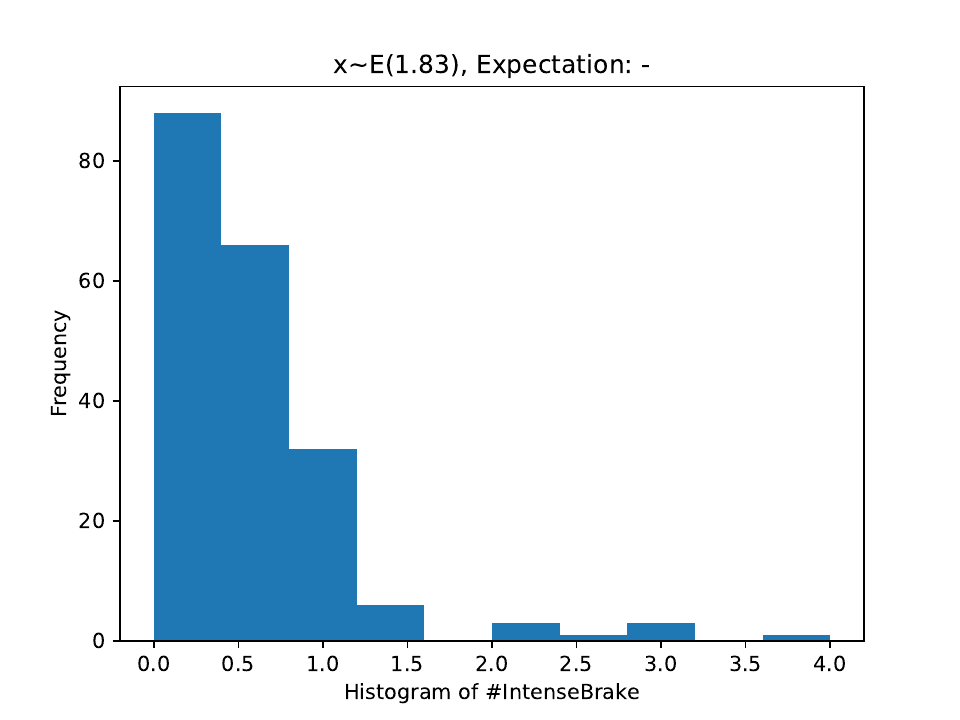}
  }
\subfigure[] {
      \includegraphics*[width=1.6in]{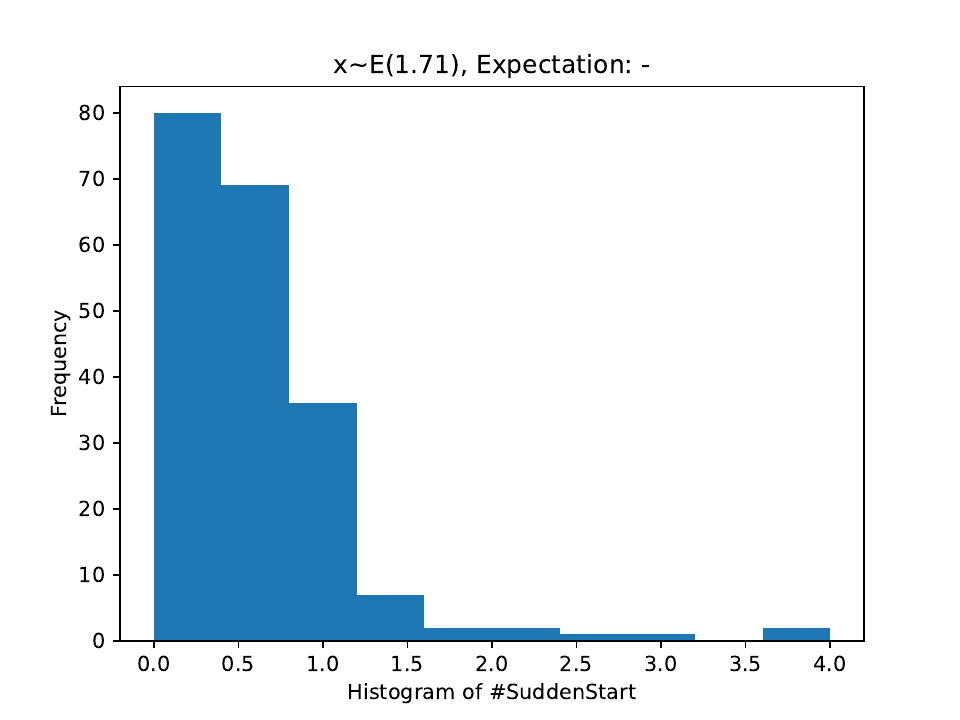}
  }
\subfigure[] {
      \includegraphics*[width=1.6in]{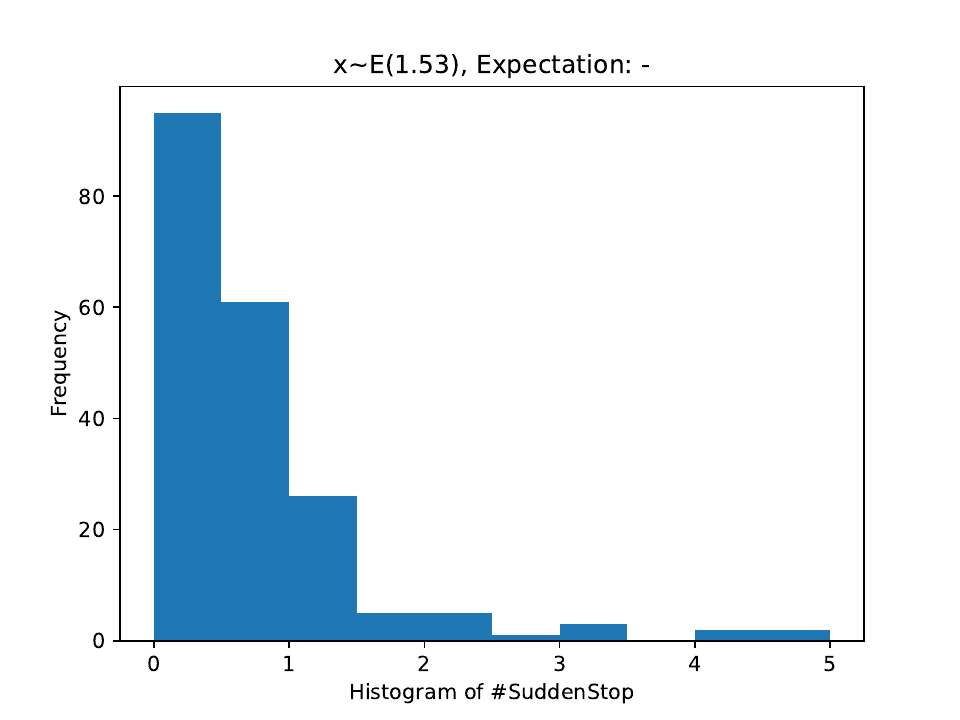}
  }
\caption{The overall distributions of the metrics in the virtual UBI data.}
\label{fig:hosf_ubi}
\end{figure*}

\begin{figure*}[htbp]
\centering
\subfigure[] {
      \includegraphics*[width=1.6in]{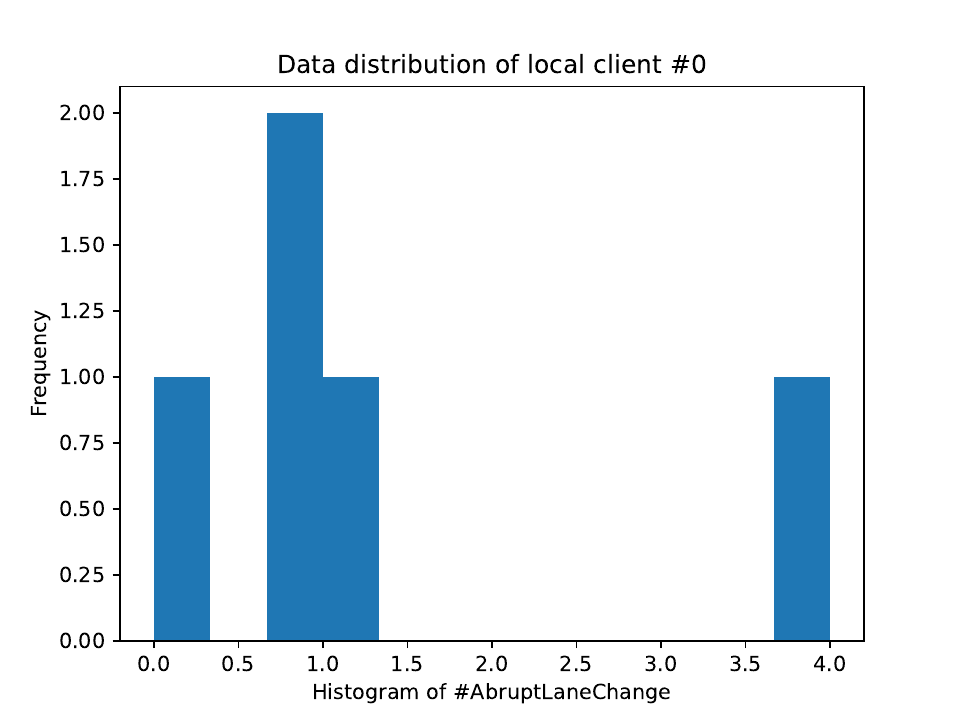}
  }
\subfigure[] {
      \includegraphics*[width=1.6in]{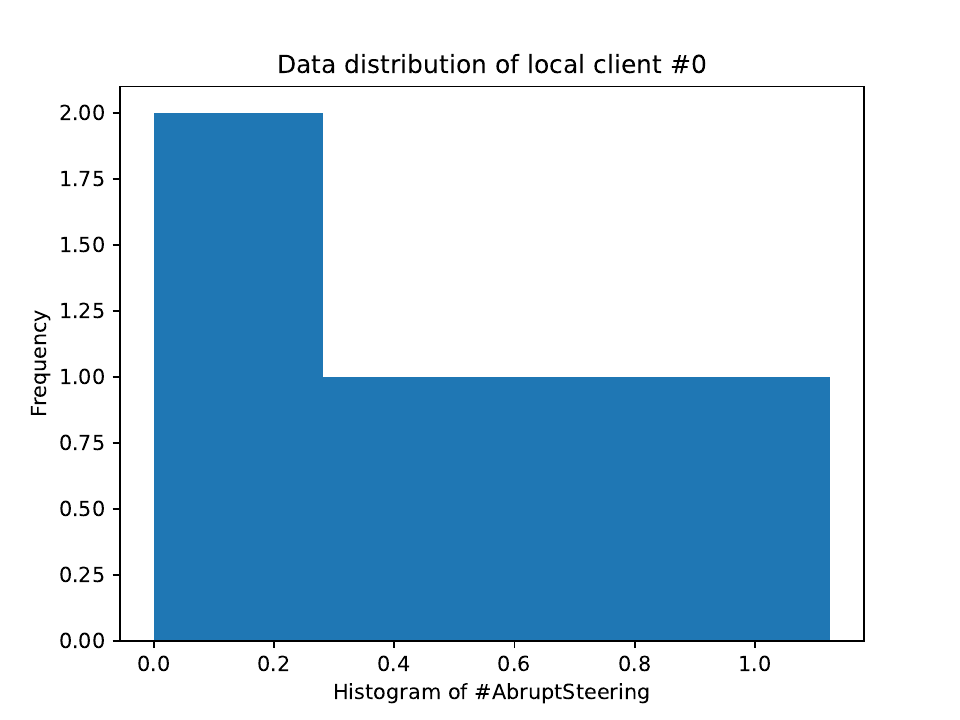}
  }
\subfigure[] {
      \includegraphics*[width=1.6in]{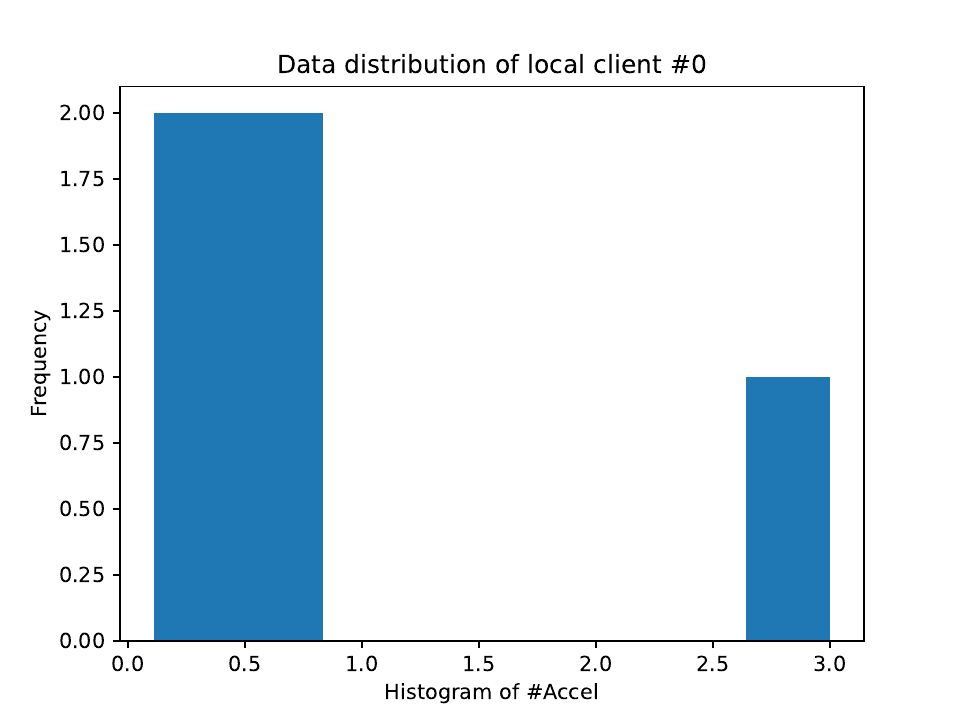}
}
\subfigure[] {
      \includegraphics*[width=1.6in]{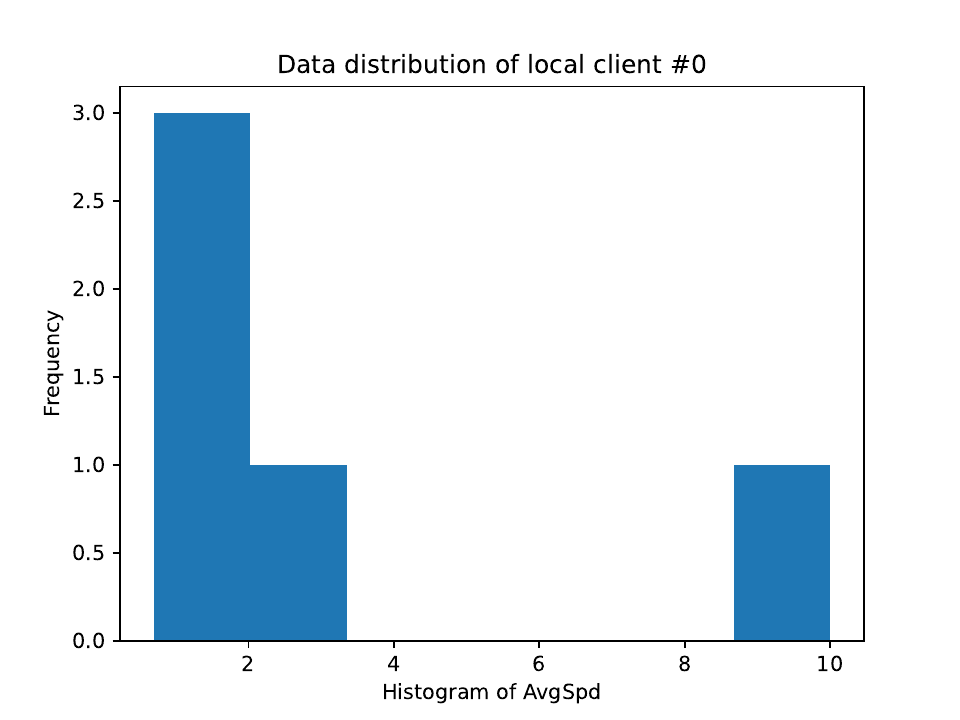}
  }
\subfigure[] {
      \includegraphics*[width=1.6in]{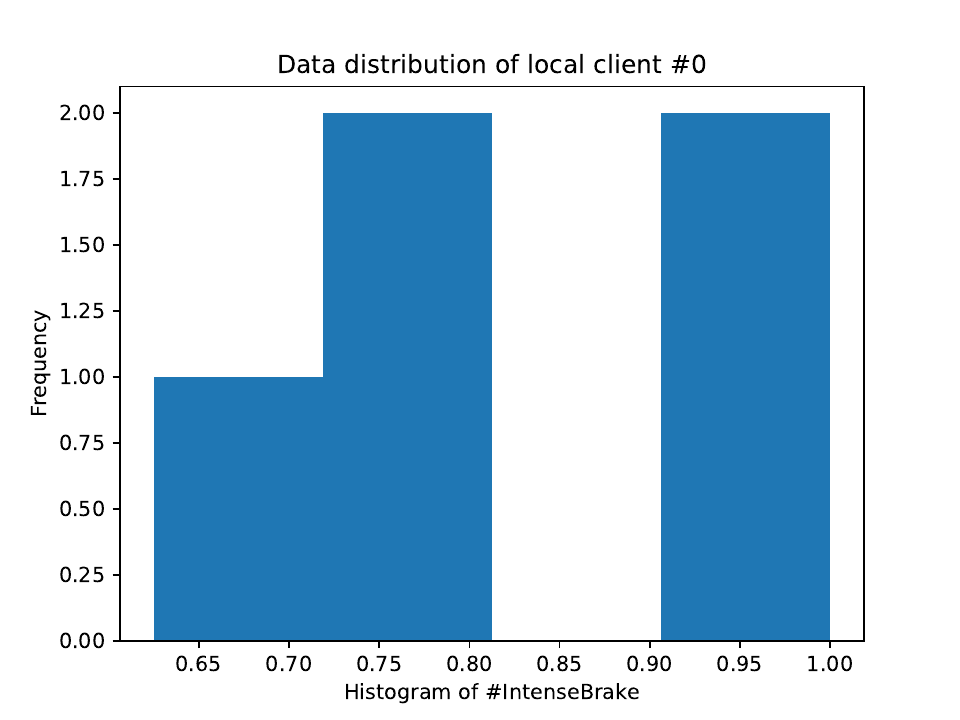}
  }
\subfigure[] {
      \includegraphics*[width=1.6in]{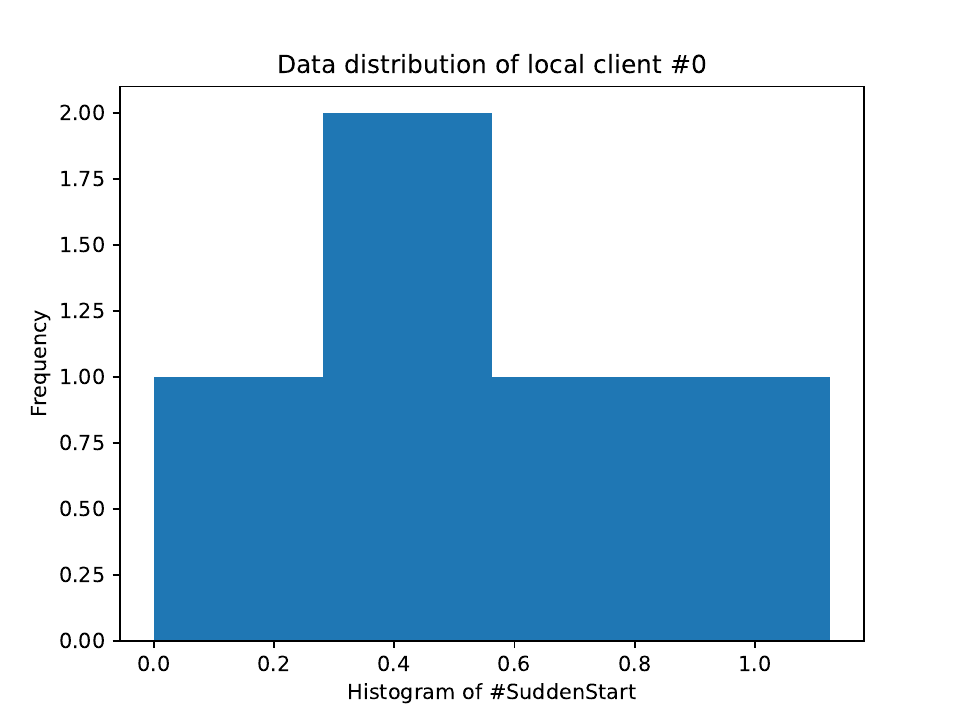}
  }
\subfigure[] {
      \includegraphics*[width=1.6in]{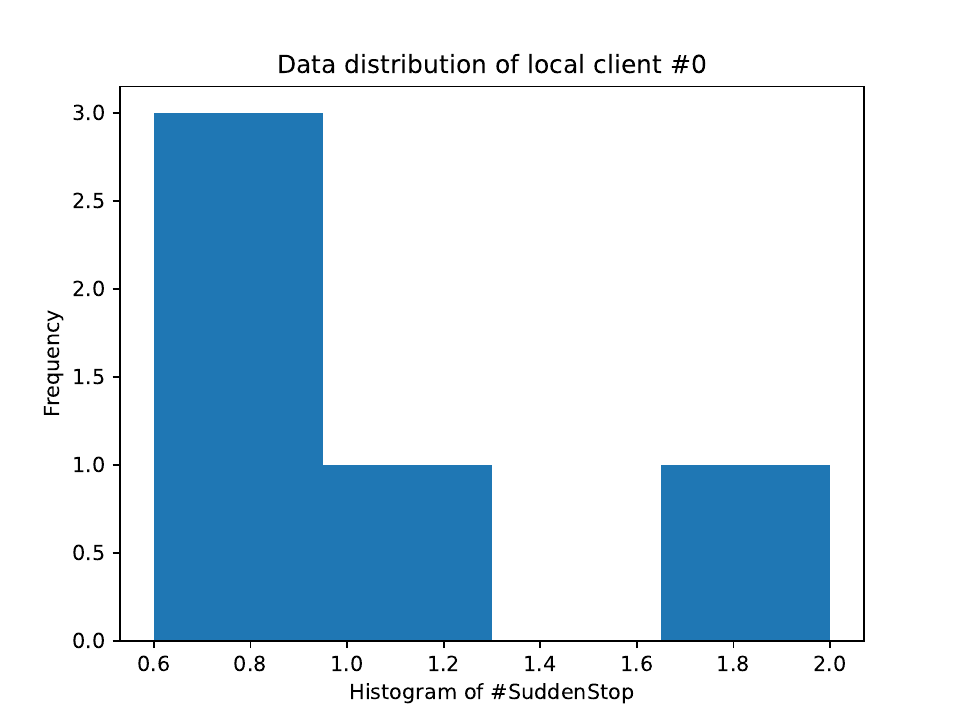}
  }
\caption{A client-level histograms of metrics in the virtual UBI data.}
\label{fig:client_hosf_ubi}
\end{figure*}

\subsubsection{Convergence Validation}
This section mainly demonstrates the convergence of the mean, variance and weights of the 7 evaluation metrics estimated by our method. Under the two client selection rates, the changing trend of the estimated mean and variance values of the metrics with the number of rounds during the federated training process is shown by the dotted line in Fig.~\ref{fig:ubi_mean_var}. Compared to the solid line of the corresponding color (the mean value calculated by CL), it can be seen that the mean lines of these metrics have achieved good convergence within 500 rounds. The two subgraphs on the right of Figure~\ref{fig:ubi_mean_var} illustrate the estimated variance of the metrics obtained by CF4CRITIC-DM with the number of rounds under two different client selection rates. It is evident that the estimated metric variances converge.

\begin{figure*}[!htbp]
\centering
\subfigure[$\tau=0.1$] {
      \includegraphics*[width=1.6in]{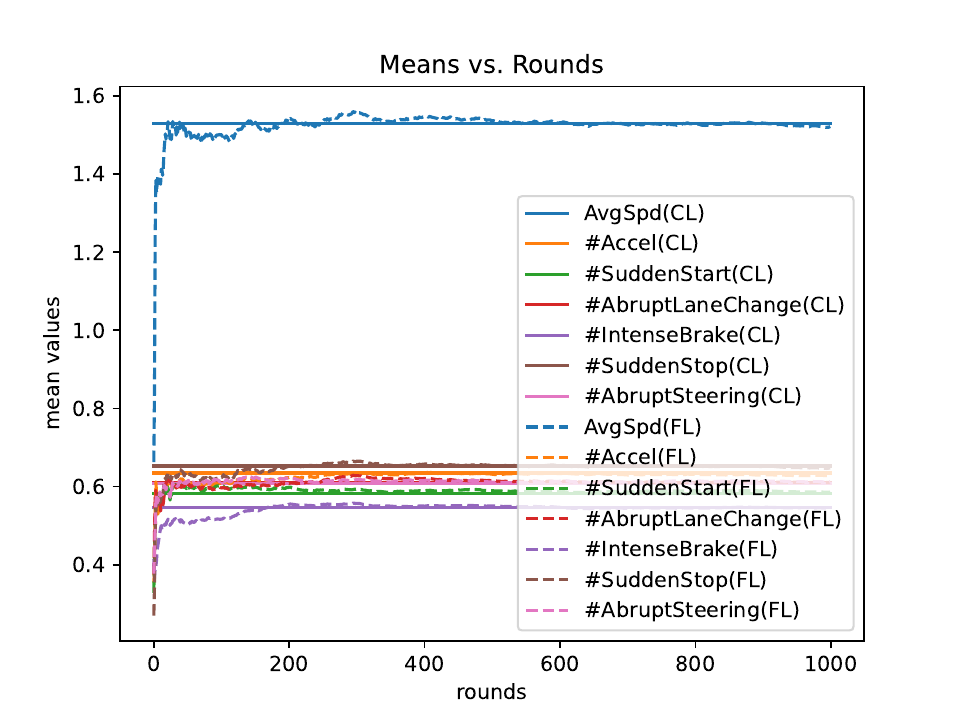}
}
\subfigure[$\tau=0.5$] {
      \includegraphics*[width=1.6in]{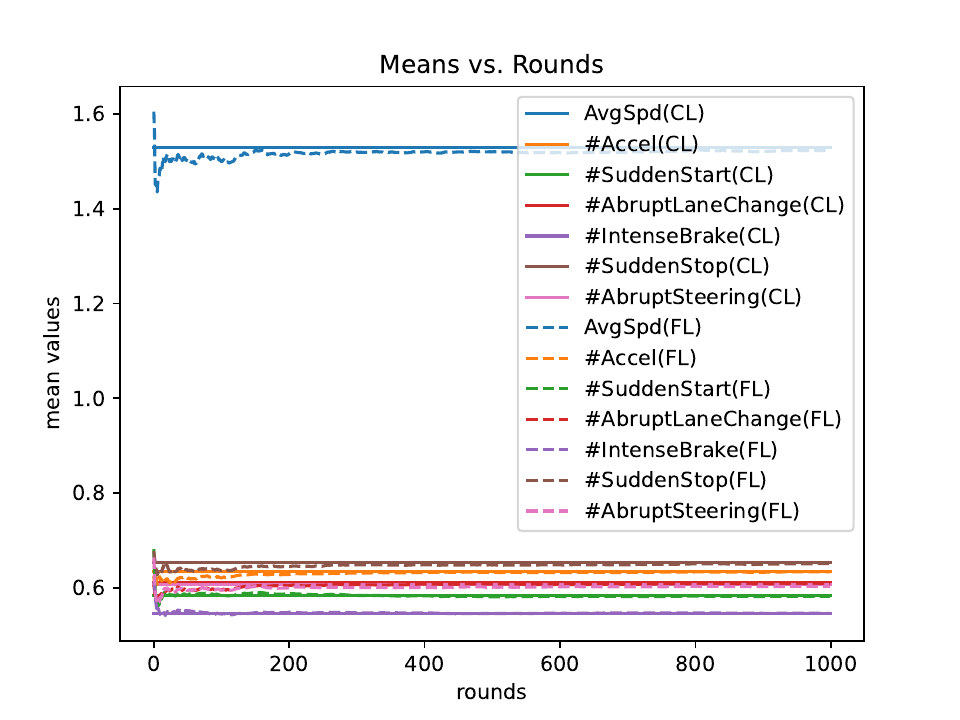}
  }
\subfigure[$\tau=0.1$] {
      \includegraphics*[width=1.6in]{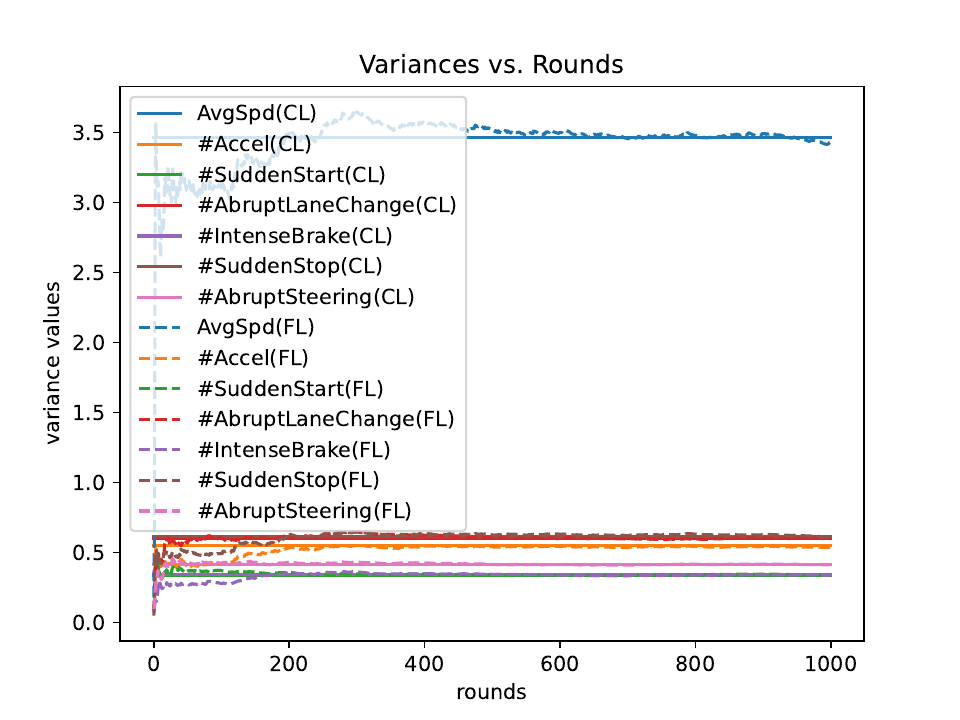}
}
\subfigure[$\tau=0.5$] {
      \includegraphics*[width=1.6in]{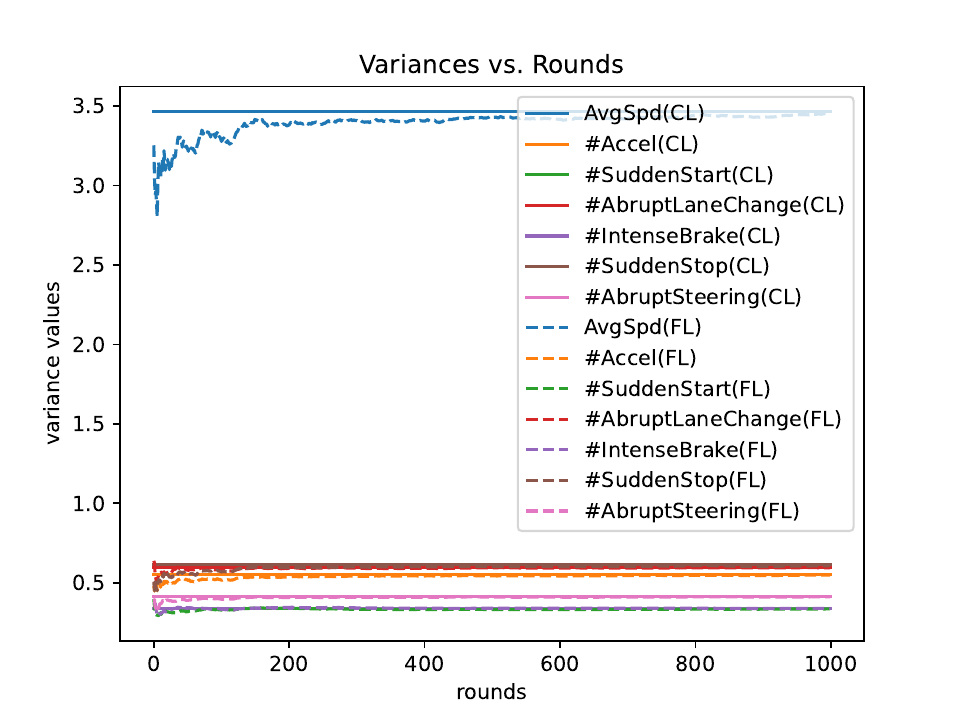}
  }
\caption{The estimated mean and variance of metrics in the virtual UBI data.}\label{fig:ubi_mean_var}
\end{figure*}

We then studied the metric weights. The two graphs in Fig.~\ref{fig:ubi_weights} demonstrate the alteration in the estimated weights with the number of rounds for two different client selection rates. It is clear that the 7 dotted lines rapidly reach their target lines, regardless of the selection rate being 0.1 or 0.5. The only difference is that the estimated weights converge slower when $\tau$ is 0.1.

\begin{figure}[!htbp]
\centering
\subfigure[$\tau=0.1$] {
      \includegraphics*[width=1.6in]{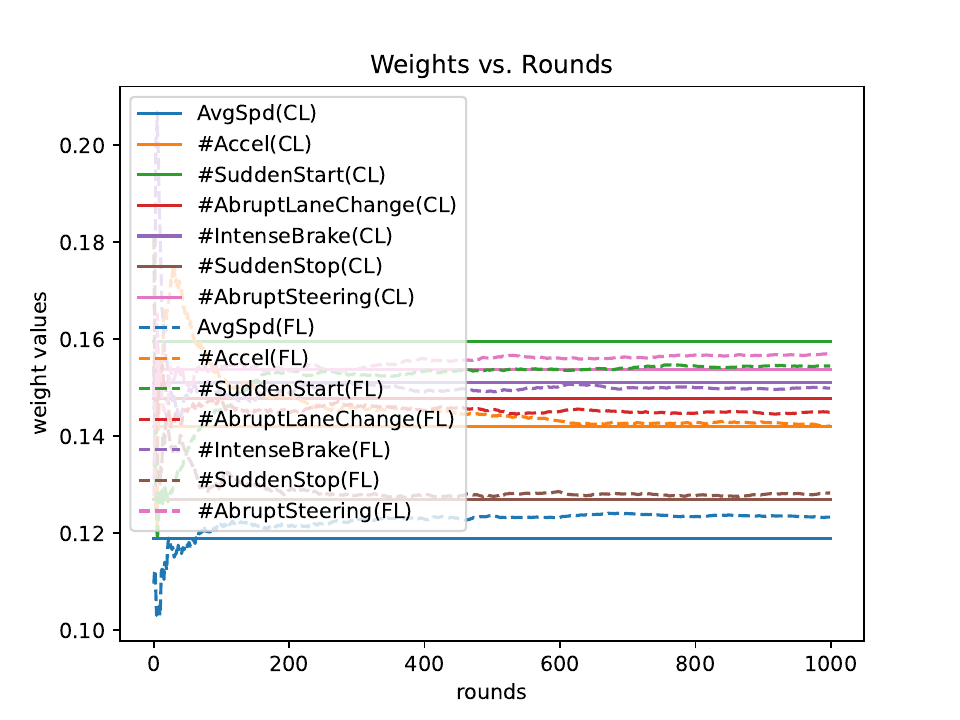}
}
\subfigure[$\tau=0.5$] {
      \includegraphics*[width=1.6in]{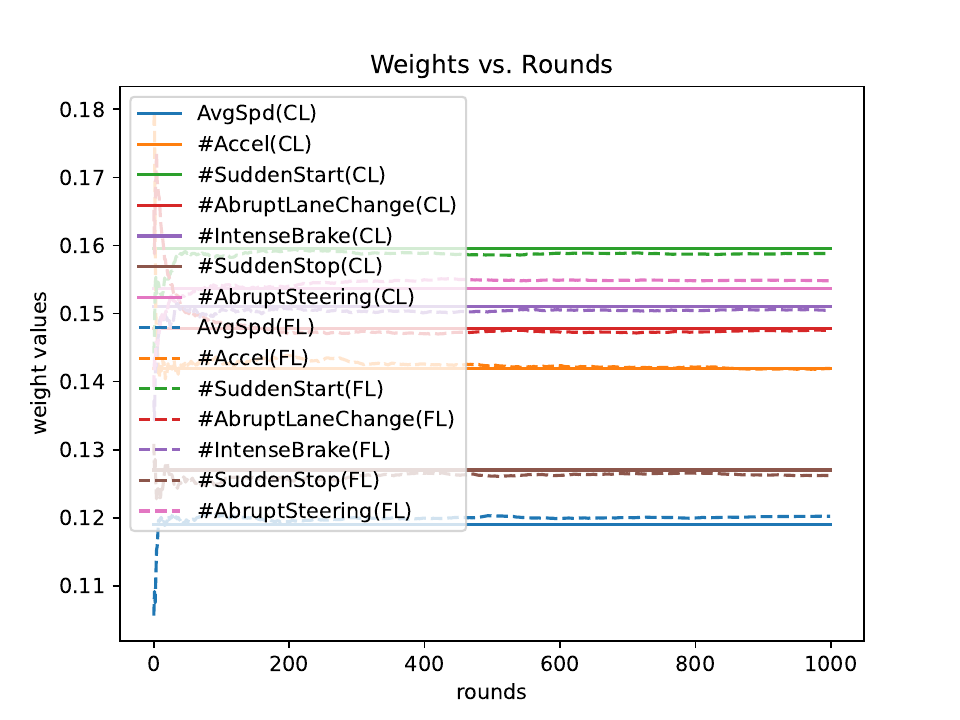}
  }
\caption{The estimated weights of the metrics in the virtual UBI data.} \label{fig:ubi_weights}
\end{figure}

\subsubsection{Utility Consistency Validation}
In this section, we take the metric weights as an example to validate the consistency of the model parameters. The CRITIC weight vector in $\mathcal{F}$ and the weight given by the two federated scoring models are listed in Table~\ref{tab:ubi_metric_weight}. 

\begin{table*}[]\small
\caption{Comparison of metric weight from different scoring models on virtual UBI data}
\label{tab:ubi_metric_weight}
\centering
\begin{tabular}{l|c|c|c|c|c|c|c}
\hline
\textbf{Models} & \textbf{AvgSpd} & \textbf{\#Accel} & \textbf{\#SuddenStart} & \textbf{\#AbruptLaneChange} & \textbf{\#IntenseBrake} & \textbf{\#SuddenStop} & \textbf{\#AbruptSteering} \\ \hline
$\mathcal{F}$& 0.1190&0.1420&	0.1595&0.1478&0.1511&0.1270&0.1537 \\
$\hat{\mathcal{F}}_{|\tau=0.1}$ & 0.1228&0.1418&0.1563&0.1435&0.1488&0.1283&0.1584 \\
$\hat{\mathcal{F}}_{|\tau=0.5}$ & 0.1202&0.1423&0.1587&0.1472&0.1504&0.1265&0.1546 \\ \hline
\end{tabular}
\end{table*}

It can also be seen that $\hat{\mathcal{F}}_{|\tau=0.1}$ and $\hat{\mathcal{F}}_{|\tau=0.5}$ give weight values very close to those of $\mathcal{F}$. To further prove the advantages of CRITIC weights, we first analyze the correlation between these evaluation metrics. Table~\ref{tab:ubi_pearson} lists the Pearson's correlation coefficients among these 7 metrics, where the coefficients with strong correlation are marked in bold. As can be noted from the table: 1) metrics with low correlation are given higher weights. For example, the metric $\#$SuddenStart has the highest value; 2) conversely, metrics with higher correlations are given lower weights. For example, the AvgSpd metric has a relatively high correlation with $\#$AbruptLaneChange, $\#$SuddenStop, $\#$Accel and $\#$AbruptSteering, so its weight value is also the lowest.

\begin{table*}[]\small
\caption{Pearson correlation coefficient matrix between evaluation metrics of virtual UBI data}
\label{tab:ubi_pearson}
\centering
\scalebox{0.8}{
\begin{tabular}{|l|c|c|c|c|c|c|c|}
\hline
\multicolumn{1}{|c|}{\textbf{}} & \textbf{AvgSpd} & \textbf{\#Accel} & \textbf{\#SuddenStart} & \textbf{\#AbruptLaneChange} & \textbf{\#IntenseBrake} & \textbf{\#SuddenStop} & \textbf{\#AbruptSteering} \\ \hline
\textbf{AvgSpd}                 & 1.0             & \textbf{0.69}    & 0.53                   & \textbf{0.73}               & 0.57                    & \textbf{0.76}         & \textbf{0.63}             \\ \hline
\textbf{\#Accel}                & \textbf{0.69}   & 1.0              & 0.31                   & 0.42                        & 0.38                    & 0.57                  & 0.45                      \\ \hline
\textbf{\#SuddenStart}          & 0.53            & 0.31             & 1.0                    & 0.26                        & 0.25                    & 0.48                  & 0.51                      \\ \hline
\textbf{\#AbruptLaneChange}     & \textbf{0.73}   & 0.42             & 0.26                   & 1.0                         & 0.49                    & 0.5                   & 0.42                      \\ \hline
\textbf{\#IntenseBrake}         & 0.57            & 0.38             & 0.25                   & 0.49                        & 1.0                     & 0.52                  & 0.34                      \\ \hline
\textbf{\#SuddenStop}           & \textbf{0.76}   & 0.57             & 0.48                   & 0.5                         & 0.52                    & 1.0                   & 0.48                      \\ \hline
\textbf{\#AbruptSteering}       & \textbf{0.63}   & 0.45             & 0.51                   & 0.42                        & 0.34                    & 0.48                  & 1.0                       \\ \hline
\end{tabular}}
\end{table*}

We then compare the consistency of the utility of the 3 scoring models. From the perspective of the score distribution, Fig.~\ref{fig:ubi_score_compare} shows the histograms of the scoring results of different models, in which Fig.~\ref{fig:ubi_score_comparea} first draws the histogram of the scores from $\mathcal{F}$, and the two subgraphs on the right are the histograms of $\hat{\mathcal{F}}_{|\tau=0.1}$ and $\hat{\mathcal{F}}_{|\tau=0.5}$, respectively. Once again, our two models present an identical distribution of scores for the target model $\mathcal{F}$. 

\begin{figure}[!htbp]
\centering
\subfigure[$\mathcal{F}$] {
\label{fig:ubi_score_comparea}
      \includegraphics*[width=1.0in]{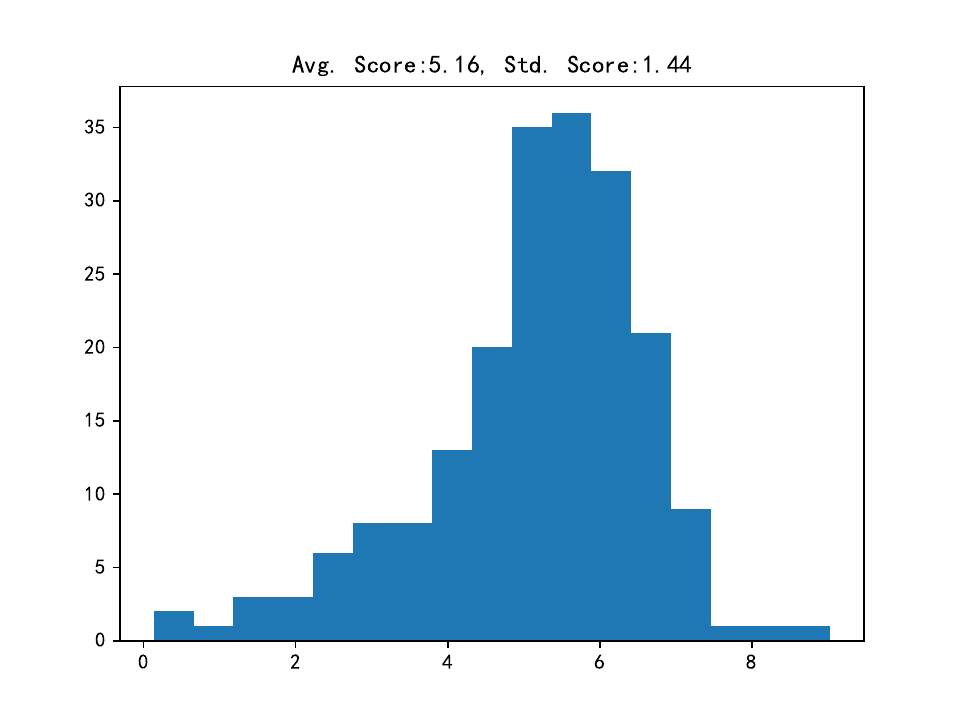}
}
\subfigure[$\hat{\mathcal{F}}_{|\tau=0.1}$] {
\label{fig:ubi_score_compareb}
      \includegraphics*[width=1.0in]{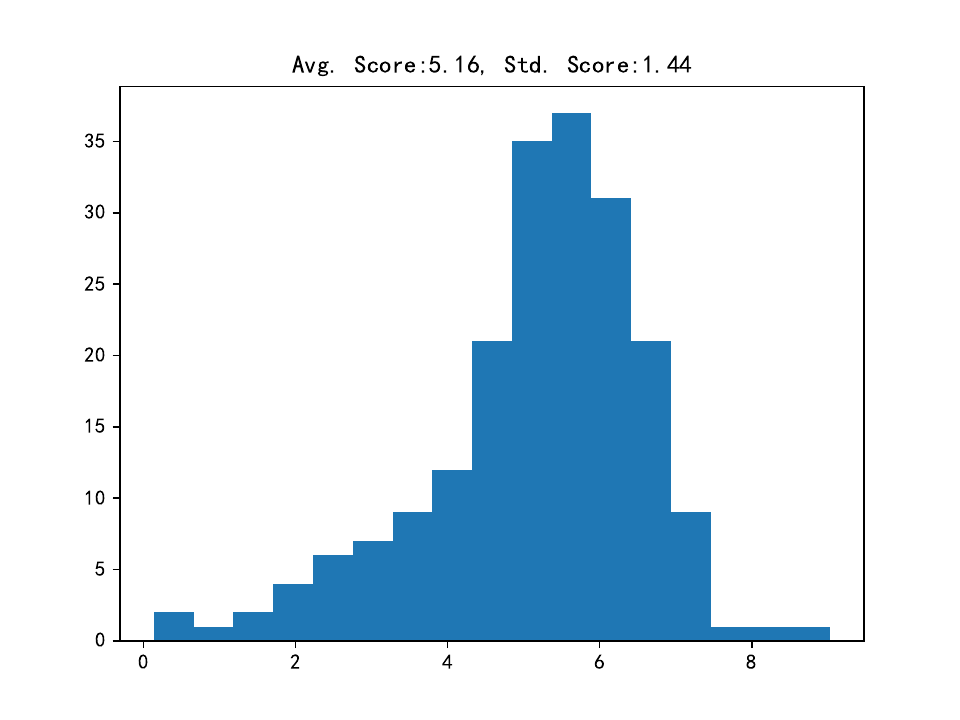}
  }
\subfigure[$\hat{\mathcal{F}}_{|\tau=0.5}$] {
\label{fig:ubi_score_comparec}
      \includegraphics*[width=1.0in]{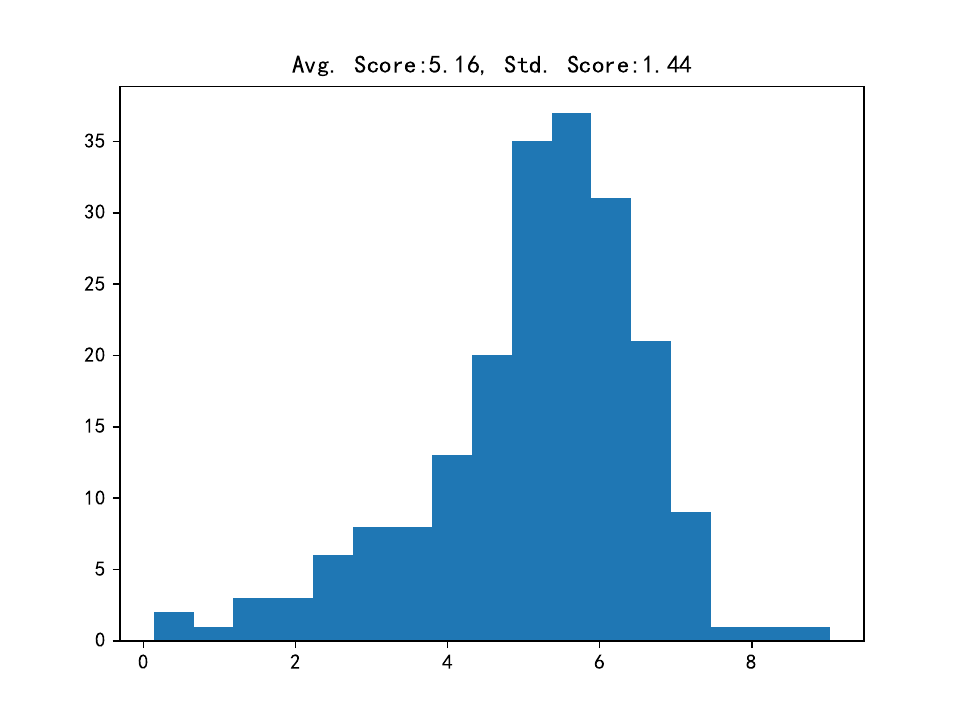}
}
\caption{Comparison of trip score distributions on the virtual UBI data.}
\label{fig:ubi_score_compare}
\end{figure}

The regression indexes used above are also used for quantitative consistency verification. Specifically, the scoring result of $\mathcal{F}$ is used as the ground truth, and the difference between it and the scoring results of the three FL scoring models is investigated. Table~\ref{tab:ubi_score_val} lists these indexes. First, we note that the CRITIC-DM model, that is, $\text{FedAvg}_{|\tau=0.5}$, learned by the typical FedAvg method, again fails to reach consistency with the target model $\mathcal{F}$. It can be clearly seen that the difference between the model scoring results of our method and the target model is very small, and the utility loss is negligible.

\begin{table}[]\small
\caption{Model score validation indexes on virtual UBI Trips}
\label{tab:ubi_score_val}
\centering
\begin{tabular}{c|l|l|l|l}
\hline
\textbf{} & $\mathcal{F}$ & $\hat{\mathcal{F}}_{|\tau=0.1}$ & $\hat{\mathcal{F}}_{|\tau=0.5}$ & $\text{FedAvg}_{|\tau=0.5}$ \\ \hline
MSE       & 0                      & 0.0006                & 0.0000    & 0.1990            \\ 
MAE       & 0                      & 0.0186                & 0.0024     & 0.3627           \\ 
RMSE      & 0                      & 0.0200            & 0.0051     & 0.4461      \\ 
$R^2$-score        & 1.0                    & 0.9998            & 0.9999  & 0.8929         \\ \hline
\end{tabular}
\end{table}

\subsubsection{Objective Fairness Analysis}
In this section, we illustrate the superiority of our method by comparing it with the subjective scores in this data set. First, Fig.~\ref{fig:ubi_score_subject_object} shows the distribution of subjects' subjective scores and our model's objective scores. It can be seen that the subjective rating is more in line with the standard normal distribution, while $\hat{\mathcal{F}}_{|\tau=0.5}$ produces a biased normal distribution. From Fig.~\ref{fig:ubi_score_subject_objecta}, we can see that most trips are rated between 6 and 8 points, and there are few trips rated as low and high scores in Fig.~\ref{fig:ubi_score_subject_objectb}. 

Then we also calculated the MSE, MAE, RMSE and $R^2$-score indexes between the $\hat{\mathcal{F}}_{|\tau=0.5}$ score and the subjective score, which were: 3.6387 and 1.5132, 1.9075, and 0.1419, respectively. No large differences are observed in these indexes. Moreover, our model performs even better than the supervised multilayer perceptron regression and safety index methods used in the work of \cite{lopez2018genetic}.

\begin{figure}[!htbp]

\centering
\subfigure[$\mathcal{F}$] {
\label{fig:ubi_score_subject_objecta}
      \includegraphics*[width=1.6in]{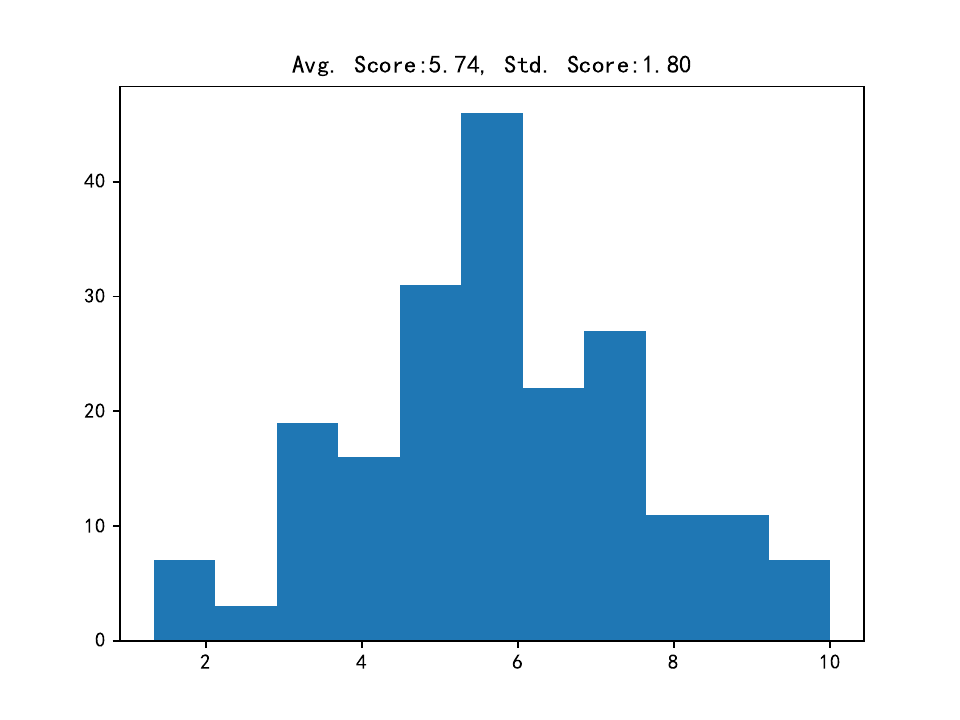}
}
\subfigure[$\hat{\mathcal{F}}_{|\tau=0.5}$] {
\label{fig:ubi_score_subject_objectb}
      \includegraphics*[width=1.6in]{ubi_results/FL_Overall_Score0.5.pdf}
}
\caption{Distributions of subjective scores and our model's objective scores.}\label{fig:ubi_score_subject_object}
\end{figure}

However, subjective scoring is not necessarily justified. Table~\ref{tab:ubi_effect_study} selects three typical trips for comparison and discussion. As shown in the table, the first four rows give the scoring results of the three trips. Followed by, we list the original evaluation metric values of each trip, and their values per kilometer after normalization based on distance. The purpose of normalizing by distance traveled is to ensure that scores are fairly scaled regardless of distance traveled. Therefore, considering those trips normalized by driving distance can achieve a more accurate and fair scoring mechanism.

\begin{table}[]\small
\caption{Examples of rationality comparison between subjective scoring and model scoring}
\label{tab:ubi_effect_study}
\centering
\scalebox{0.8}{
\begin{tabular}{|l|cc|cc|cc|}
\hline
\textbf{}                           & \multicolumn{2}{c|}{\textbf{Trip \#1}}                       & \multicolumn{2}{c|}{\textbf{Trip \#2}}                       & \multicolumn{2}{c|}{\textbf{Trip \#23}}                      \\ \hline
\textbf{Subjective scoring results} & \multicolumn{2}{c|}{4}                                       & \multicolumn{2}{c|}{6}                                       & \multicolumn{2}{c|}{9.3333}                                  \\ \hline
\textbf{}                           & \multicolumn{2}{c|}{4.5161}                                  & \multicolumn{2}{c|}{7.2339}                                  & \multicolumn{2}{c|}{3.4615}                                  \\ \hline
\textbf{}                           & \multicolumn{2}{c|}{4.5152}                                  & \multicolumn{2}{c|}{7.2353}                                  & \multicolumn{2}{c|}{3.4568}                                  \\ \hline
\textbf{}                           & \multicolumn{2}{c|}{4.5148}                                  & \multicolumn{2}{c|}{7.2288}                                  & \multicolumn{2}{c|}{3.4510}                                  \\ \hline
\textbf{Distance}                   & \multicolumn{2}{c|}{5km}                                     & \multicolumn{2}{c|}{9km}                                     & \multicolumn{2}{c|}{1km}                                     \\ \hline
\textbf{}                           & \multicolumn{1}{c|}{raw} & \multicolumn{1}{l|}{per km} & \multicolumn{1}{l|}{raw} & \multicolumn{1}{l|}{per km} & \multicolumn{1}{l|}{raw} & \multicolumn{1}{l|}{per km} \\ \hline
\textbf{AvgSpd}                     & \multicolumn{1}{c|}{8}   & 1.6                               & \multicolumn{1}{c|}{9}   & 1.0                               & \multicolumn{1}{c|}{3}   & 3.0                               \\ \hline
\textbf{\#Accel}                    & \multicolumn{1}{c|}{3}   & 0.6                               & \multicolumn{1}{c|}{5}   & 0.5556                            & \multicolumn{1}{c|}{1}   & 1.0                               \\ \hline
\textbf{\#SuddenStart}              & \multicolumn{1}{c|}{6}   & 1.2                               & \multicolumn{1}{c|}{5}   & 0.5556                            & \multicolumn{1}{c|}{1}   & 1.0                               \\ \hline
\textbf{\#AbruptLaneChange}         & \multicolumn{1}{c|}{3}   & 0.6                               & \multicolumn{1}{c|}{3}   & 0.3333                            & \multicolumn{1}{c|}{2}   & 2.0                               \\ \hline
\textbf{\#IntenseBrake}             & \multicolumn{1}{c|}{3}   & 0.6                               & \multicolumn{1}{c|}{0}   & 0.0                               & \multicolumn{1}{c|}{1}   & 1.0                               \\ \hline
\textbf{\#SuddenStop}               & \multicolumn{1}{c|}{4}   & 0.8                               & \multicolumn{1}{c|}{0}   & 0.0                               & \multicolumn{1}{c|}{0}   & 0.0                               \\ \hline
\textbf{\#AbruptSteering}           & \multicolumn{1}{c|}{5}   & 1.0                               & \multicolumn{1}{c|}{6}   & 0.6667                            & \multicolumn{1}{c|}{2}   & 2.0                               \\ \hline
\end{tabular}
}
\end{table}

Among the three trips, the subjective and objective scoring results of the trip $\#$1 are relatively close, and we observed that there are more trips with similar scores on all trips, which is why the four validation indexes mentioned above have relatively small differences. For trip $\#$23, subjects rated it highly due to its short distance and low absolute frequency of aggressive driving events. However, when these driving events are normalized by distance to the per kilometer scale, we find that most of the normalized metric values are higher than the normalized ones of the first two trips, so our model gives trip $\#$23 the lowest score. It can be seen that subjective scoring will be disturbed by the intuition of "what you see is what you get", thus lacking relative fairness when evaluating the driving performance of different drivers.

\subsection{Discussion}
Based on the experimental results, this section discusses the effectiveness of the proposed solution from the following points of view.

1. The convergence verification experiments on the two datasets show that : 1) The proposed CF4CRITIC-DM method can effectively guarantee convergence during the scoring model training process. As the number of training rounds increases, the globally aggregated statistics and weights of metrics (i.e., model parameters) are convergent; 2) Client selection rate affects the convergence speed. The $T$ required for the global model to reach convergence is in a cross-correlation relationship with the client selection rate. Furthermore, the experimental results strongly support Theorem 4.1 and Theorem 4.2 in this paper.

2. The main goal of the proposed method is to reduce the loss of utility of the federated scoring model. From the utility consistency verification results on the two datasets, it is obvious that the proposed CF4CRITIC-DM method can obtain a federated scoring model that is highly consistent with the CL-based scoring model.

3. Finally, in terms of the application effect of the proposed scoring model, this study proves the proposed method from two aspects: 1) In the absence of prior knowledge(labels), it can still effectively distinguish different driving performances of drivers; 2) It does not require too much manual intervention and has better fairness than subjective scoring.

\section{Conclusion}
Driver profile is a trending task in intelligent transportation systems, where driving performance is usually given a score as a profile. The popularity of intelligent connected vehicles allows this task to collect large-scale driving data to build a scoring model. However, the lack of target labels and data privacy concerns requires an unsupervised and privacy-friendly solution.
 
This paper presents a FedDriveScore framework to compensate for the drawback of the data-centralized scoring method. First, by treating the final score as a mixture density derived from the evaluation metrics, a CRITIC-DM method is introduced to construct a scoring model on unlabeled driving trip data. Then, 
a consistently federated version of the CRITIC-DM method is proposed: 1) by combining homomorphic encryption and federated learning to deal with privacy, regulations, and other issues in the centralized learning process. 2) to ensure the effectiveness of the federated scoring model suffered from the statistical heterogeneity of local driving data on the vehicle side.

This method is tested on two datasets from two application views. The results show that our method ensures that the training process converges to the CL level and that the federated scorecard model can achieve lossless performance in scoring driving performance.

\ifCLASSOPTIONcompsoc
  \section*{Acknowledgments}
\else
  \section*{Acknowledgment}
\fi

\bibliographystyle{unsrtnat}
\bibliography{IEEEabrv,sample_library.bib}


\end{document}